\documentclass{article}

\usepackage[preprint]{neurips_2026}

\usepackage[utf8]{inputenc}
\usepackage[T1]{fontenc}
\usepackage{hyperref}
\usepackage{url}
\usepackage{microtype}
\usepackage{graphicx}
\usepackage{subcaption}
\usepackage{booktabs}
\usepackage{array}
\usepackage{tabularx}
\usepackage{amsfonts}
\usepackage{nicefrac}
\usepackage{xcolor}
\usepackage{bm}
\usepackage{algorithm, algpseudocode}

\newcolumntype{P}[1]{>{\raggedright\arraybackslash}p{#1}}
\newcolumntype{Y}{>{\raggedright\arraybackslash}X}

\usepackage{amsmath}
\usepackage{amssymb}
\usepackage{mathtools}
\usepackage{amsthm}

\usepackage[capitalize,noabbrev]{cleveref}

\theoremstyle{plain}
\newtheorem{theorem}{Theorem}[section]
\newtheorem{proposition}[theorem]{Proposition}

\theoremstyle{definition}

\theoremstyle{remark}

\usepackage[disable,textsize=tiny]{todonotes}

\title{Integrating Neural Encoders in Bayesian Generalized Linear Mixed Models for Multimodal Data}

\author{%
  \bfseries Yuankang Zhao$^{1}$, Youngsoo Baek$^{1}$, Felipe A. Medeiros$^{2}$, \\
  \bfseries Samuel Berchuck$^{1}$, Matthew M. Engelhard$^{1}$ \\[4pt]
  $^{1}$Department of Biostatistics \& Bioinformatics, Duke University \\
  $^{2}$Department of Ophthalmology, University of Miami \\
  \texttt{\{yuankang.zhao, m.engelhard\}@duke.edu}
}

\begin{document}

\maketitle

\begin{abstract}
Scalable Bayesian inference for generalized linear mixed models (GLMMs)
provides uncertainty-aware analysis of correlated longitudinal data, but
existing scalable approaches largely assume low-dimensional tabular predictors
and do not directly accommodate high-dimensional modalities such as images and text. We address this limitation by learning one or more
modality-specific neural encoders jointly with a GLMM objective, then
performing variance-corrected stochastic-gradient MCMC for the GLMM parameters
conditional on the learned representation. This conditional-Bayes design
combines supervised representation learning with posterior uncertainty
quantification for population-level effects, subject-specific heterogeneity,
and modality-level random slopes. The resulting model preserves interpretable
fixed and random effects for structured covariates and learned modalities while
scaling gracefully to large longitudinal datasets. In simulation studies, our method
recovers posterior means and variance estimates from full-data MCMC
benchmarks after covariance correction. We further evaluate uncertainty through
parameter-level interval coverage in simulations and predictive calibration on
held-out data. Applications to glaucoma progression and adolescent mental
health demonstrate that the framework 
allows nuanced assessment of the relative importance of each modality on both individual and population levels without
sacrificing predictive performance.
\end{abstract}
\section{Introduction}
Bayesian Generalized Linear Mixed Models (GLMMs) 
are commonly used to analyze clustered and longitudinal outcomes
\citep{breslow1993approximate}
in clinical medicine and other high-stakes settings requiring uncertainty quantification (UQ) for both model parameters and predictions \citep{lopez2025uncertainty}. 
However, traditional Bayesian GLMMs are designed for low-dimensional tabular data and cannot accommodate complex, multi-modal data (\textit{e.g.}, text, images) commonly encountered in healthcare 
\citep{lindstrom1990nonlinear}.  
Thus, there is a need for methodologies that integrate multi-modal observations while accounting for within-subject correlations \citep{acosta2022multimodal} 
and providing
principled uncertainty quantification. 
For example, in longitudinal glaucoma care, the information most predictive of future progression risk may differ substantially across patients. As illustrated in Figure~\ref{fig:glaucoma_motivation}, objective evidence from retina scans may dominate prediction for some patients, whereas patient-reported symptoms may be more informative for others, even when both patients are observed through the same set of modalities. A useful model should therefore not only combine heterogeneous inputs, including high-dimensional imaging and structured clinical covariates, but also quantify patient-specific variation in the importance of each modality. Because such individualized assessments may guide clinical interpretation, the model should further represent uncertainty in both modality-specific effects and predicted risk, rather than returning only point estimates.

\begin{figure}[t]
    \centering
    \includegraphics[width=\textwidth]{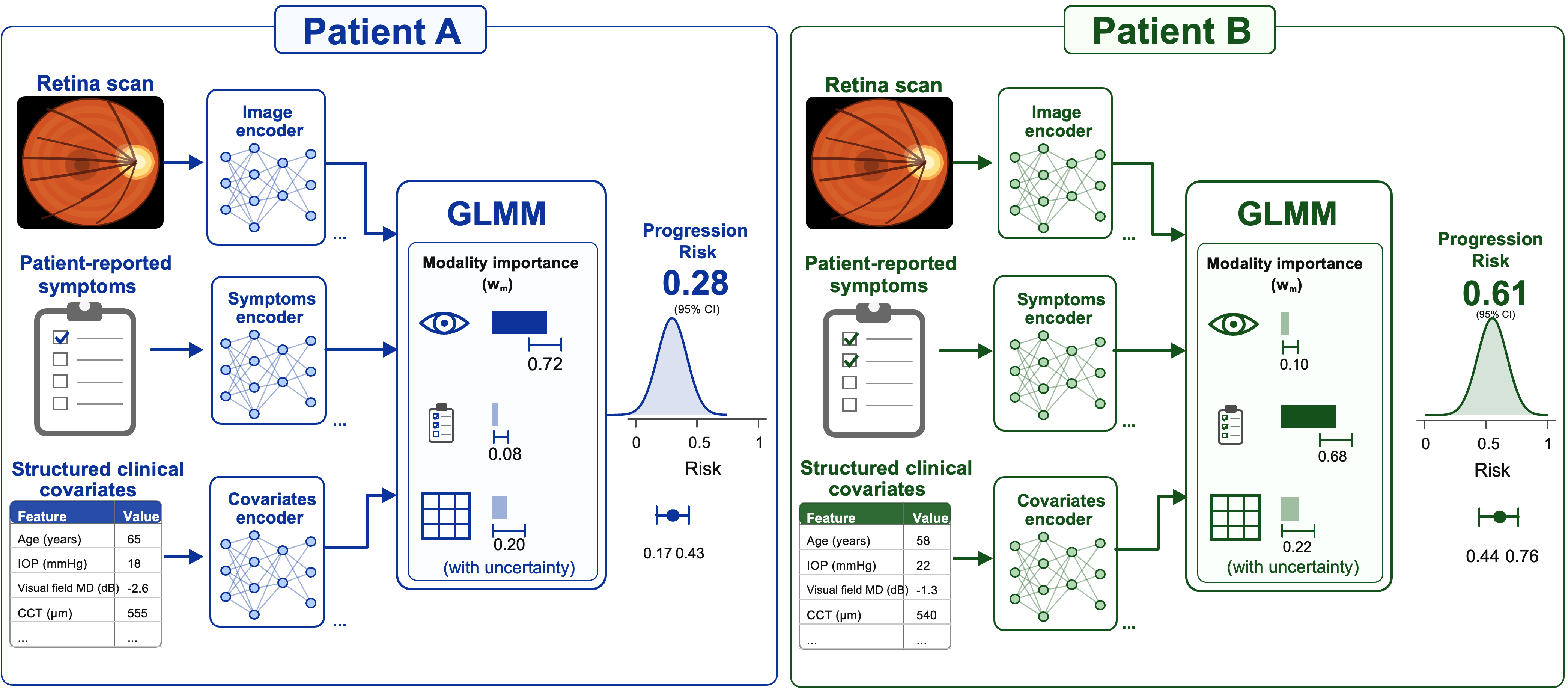}
    \caption{
    Conceptual illustration of patient-specific modality importance in longitudinal glaucoma modeling.
    Different modalities may dominate prediction for different patients: retina imaging is most informative for Patient A, whereas patient-reported symptoms are most informative for Patient B.
    The framework aims to estimate these individual-level modality contributions, together with uncertainty in both modality importance and predicted progression risk.
    }
    \label{fig:glaucoma_motivation}
\end{figure}

Addressing these needs presents two primary challenges. The first 
challenge is to integrate neural networks into the GLMM framework while retaining interpretability and accommodating heterogeneous covariates. The ideal model for longitudinal multimodal data would employ modality-specific encoders (e.g., CNNs for imaging, Transformers for text) alongside a standard linear predictor for demographics and other tabular features \citep{gao2024explainable}. This design
allows us to quantify the importance of each modality on both population and subject-specific levels, facilitating interpretation.
However, existing hybrid approaches often sacrifice this granularity, either by embedding neural networks solely as global fixed-effect predictors or by burying subject-specific effects within deep architectures, thereby obscuring the distinction between population-level trends and individual deviations \citep{simchoni2023integrating,shi2022generalized,mandel2023neural}.

The second challenge concerns the scalability of inference. Integrating neural networks into GLMMs results in a highly non-convex loss function and an intractable marginal likelihood, complicating parameter estimation. Existing strategies sacrifice rigorous UQ to obtain computational scalability. For example, point estimation methods (\textit{e.g.}, Maximum A Posteriori (MAP)), are computationally efficient but fail to capture epistemic uncertainty \citep{gal2016dropout}. Conversely, approximate Bayesian methods like variational inference often rely on restrictive distributional assumptions (e.g., Gaussianity)\citep{blundell2015weight} that may ill-fit the complex posterior distributions of deep neural networks. Developing a fully Bayesian approach that scales to big data, extracts features from complex modalities, and accurately quantifies uncertainty remains an open problem.

In this paper, we develop and rigorously analyze a method allowing scalable Bayesian inference in hierarchical settings where observations span multiple complex modalities.
Our specific contributions are threefold: (1) \textbf{Neural-encoder GLMM architecture:} We 
integrate modality-specific Neural Network encoders into the GLMM framework, allowing 
representations of high-dimensional modalities to be learned jointly with corresponding population and subject-specific effects that quantify their importance.
(2) \textbf{Scalable Bayesian Inference via SGLD:} 
We propose a Stochastic Gradient Langevin Dynamics (SGLD)-based learning procedure that supports scalable inference in high-dimensional settings. 
Our two-stage procedure first learns task-adapted representations jointly with the GLMM objective, then performs posterior sampling of GLMM parameters in the stabilized, learned feature space. (3) \textbf{Empirical Validation and Performance:} Through extensive simulation studies and real-world applications to glaucoma and mental health, we demonstrate that our method offers precise posterior approximations, allows us to interpret the importance of a given modality within and across subjects, and provides accurate UQ
all without sacrificing predictive performance.

\section{Related work}
\subsection{Neural networks in mixed models}
Recent work seeks to combine the predictive power of neural networks with the structured dependence modeling of GLMMs. One approach is to embed a neural network as a global, population-level predictor while retaining traditional random effects. For instance, \citet{shi2022generalized} propose GDMix, replacing the linear fixed-effects component with a nonlinear model (e.g., a DNN). For scalability, they rely on MAP estimation
, which does not yield posterior uncertainty; this can be a limitation when deploying the model in high-stakes domains such as clinical decision-making, where epistemic uncertainty is desirable.

Other approaches attempt to integrate random effects more deeply into the neural architecture. \citet{tschalzev2024enabling} introduce Monte Carlo Generalized Mixed Effects Neural Networks (MC-GMENN), which combine NN outputs and subject-specific effects within an activation function. While improving flexibility, this deep integration can diminish interpretability, as the clear separation between population trends and subject-specific deviations is often lost. Similarly, \citet{mandel2023neural} propose Neural Mixed Effects (NME) models, allowing subject-specific parameters to influence various layers of the neural architecture. Although this maximizes flexibility, it does not explicitly address the need for modality-specific encoding in multimodal EHR contexts, where preserving a standard linear component for structured covariates (e.g., age, sex) is often required for valid statistical inference.

\subsection{Scalable Bayesian inference for GLMMs}
The intractability of the marginal likelihood in GLMMs poses a significant hurdle for Bayesian inference. \citet{tran2020bayesian} applied the Fisher identity to GLMMs within a Variational Inference (VI) framework, estimating gradients via backpropagation. However, this approach typically relies on a Gaussian variational family, often under a mean-field assumption. Such assumptions can be too restrictive for complex posterior landscapes, leading to underestimated uncertainty and limited posterior expressiveness \citep{blei2017variational}.

This limitation becomes more pronounced when adding neural networks.
Practical approximate Bayesian neural network methods can fail to preserve core Bayesian updating properties, especially when crude approximations are applied to high-dimensional neural-weight posteriors \citep{pituk2025bayesian}. This motivates a more cautious division of labor in our model: rather than making strong Bayesian claims about the full neural network weight posterior, we use the neural network primarily as a supervised encoder and perform Bayesian uncertainty quantification in the lower-dimensional GLMM layer, whose population-level and subject-level parameters remain directly interpretable.

Stochastic Gradient MCMC (SG-MCMC) methods 
offer scalability without the bias introduced by fixed variational families. \citet{berchuck2026scalablebayesianinferencegeneralized} proposed a scalable Bayesian inference algorithm for standard GLMMs based on Stochastic Gradient Langevin Dynamics (SGLD) \citep{welling2011bayesian}. Their method combines a Fisher-identity-based Monte Carlo gradient estimator with an asymptotic variance correction, achieving accurate uncertainty quantification and substantial speedups over full-data MCMC. 
However, extending this to Neural GLMMs is nontrivial: the GLMM posterior is no longer defined over fixed low-dimensional covariates, but over representations learned by a neural encoder from high-dimensional and potentially multi-modal inputs.

Our work addresses this issue by separating representation learning from Bayesian posterior sampling. The neural network is first used to learn predictive representations, after which the encoder is fixed and SG-MCMC is applied to the GLMM layer. This design ensures that the posterior target remains fixed during sampling, which is necessary for SG-MCMC diagnostics and covariance correction to be meaningful. It also aligns with our broader modeling goal: Bayesian uncertainty quantification is concentrated in the lower-dimensional, interpretable GLMM layer rather than in the full neural-weight posterior. Recent work on stochastic mirror Langevin dynamics further highlights the difficulty of sampling variance components in large-data GLMMs \citep{baek2026safe}; such constrained samplers could be incorporated into our framework as a future replacement for the variance-component update.

\section{Method}

\subsection{Model architecture}
\noindent \textbf{Problem Formulation and Feature Extraction.} 
Let $\mathcal{D} = \{(\mathbf{u}_{ij}, \mathbf{x}_{ij}, y_{ij})\}$ denote the dataset, where $i = 1, \dots, N$ indexes the subjects and $j = 1, \dots, n_i$ indexes the observations for subject $i$. We distinguish between two types of inputs: $\mathbf{u}_{ij}$ represents high-dimensional unstructured data (e.g., images) requiring deep feature extraction, and $\mathbf{x}_{ij} \in \mathbb{R}^P$ represents standard structured covariates (e.g., clinical variables) modeled with fixed effects only. The outcome variable is denoted by $y_{ij}$.

To capture high-level representations from the unstructured inputs $\mathbf{u}_{ij}$, we employ a neural network encoder $h_{\bm{\phi}}(\cdot)$ parameterized by $\bm{\phi}$. For each input $\mathbf{u}_{ij}$, the encoder produces a $K$-dimensional latent feature vector. To ensure numerical stability and consistent scaling, we standardize as follows:
\begin{equation}
    \tilde{\mathbf{h}}_{ij} = h_{\bm{\phi}}(\mathbf{u}_{ij}), \quad z_{ij,k} = \frac{\tilde{h}_{ij,k} - \mu_{h,k}}{\sigma_{h,k}},
\end{equation}
where $k \in [K]$, $z_{ij,k}$ represents the $k$-th standardized feature component, and $\mu_{h,k}, \sigma_{h,k}$ are the batch-wise mean and standard deviation. We define the augmented latent feature vector as $\mathbf{z}_{ij} = [1, z_{ij,1}, \dots, z_{ij,K}]^\top$ to include an intercept term for the random effects modeling.

\noindent \textbf{Multimodal Generalized Linear Mixed Model Framework.}
We integrate the deep feature representations into a probabilistic GLMM framework to account for both population-level trends and subject-specific heterogeneity. Assuming the conditional distribution of the response $y_{ij}$ given the random effects belongs to the exponential family, the expected value $\mu_{ij} = \mathbb{E}[y_{ij} | \bm{\gamma}_i, \mathbf{u}_{ij}, \mathbf{x}_{ij}]$ is linked to a linear predictor $\eta_{ij}$ via a monotonic link function $g(\cdot)$, $g(\mu_{ij}) = \eta_{ij}$.
The linear predictor $\eta_{ij}$ is composed of a fixed-effects component for the standard covariates and a mixed-effects component for the latent neural features: 
\begin{equation}
    \eta_{ij} = \underbrace{\mathbf{x}_{ij}^\top \bm{\alpha}}_{\text{Fixed Effects Only}} + \underbrace{\mathbf{z}_{ij}^\top (\bm{\beta} + \bm{\gamma}_i)}_{\text{Deep Mixed Effects}}.
\end{equation}
Here, $\bm{\alpha} \in \mathbb{R}^P$ denotes the fixed effects coefficients for the standard covariates $\mathbf{x}_{ij}$. The vector $\bm{\beta} \in \mathbb{R}^{K+1}$ represents the fixed effects for the latent features (representing the population-average contribution of the deep features), while $\bm{\gamma}_i \sim \mathcal{N}(\mathbf{0}, \bm{\Sigma})$ denotes the subject-specific random effects. This formulation allows the model to learn subject-specific deviations in both the baseline risk (via the intercept in $\mathbf{z}_{ij}$) and the sensitivity to specific neural features.

\textit{Application to Binary Classification.} In this work, we focus on binary classification tasks where $y_{ij} \in \{0, 1\}$. Accordingly, we model the conditional probability using a Bernoulli distribution with the canonical Logit link function $g(\mu) = \log(\frac{\mu}{1-\mu})$. The probability $p_{ij} = P(y_{ij}=1 | \bm{\gamma}_i, \dots)$ is given by:
$    p_{ij} = \sigma(\eta_{ij}) = \frac{1}{1 + \exp(-\eta_{ij})}.$
The log-likelihood for subject $i$, conditioned on the latent features and random effects, is: $\log p(y_i | \mathbf{z}_i, \mathbf{x}_i, \bm{\gamma}_i) = \sum_{j=1}^{n_i} \Big[ y_{ij} \log(p_{ij}) + (1 - y_{ij}) \log(1 - p_{ij}) \Big].$

\subsection{Identifiability and Interpretable Estimands}
\label{subsec:identifiability}

Because the latent covariates are learned, their coordinate system is not
intrinsic. For clarity, write the latent part of the linear predictor as
$\eta^{(z)}_{ij}=\tilde{\mathbf z}_{ij}^{\top}(\bm\beta+\bm\gamma_i)$, where
$\tilde{\mathbf z}_{ij}=(1,z_{ij,1},\ldots,z_{ij,K})^{\top}$. For any
nonsingular diagonal matrix $A=\mathrm{diag}(1,a_1,\ldots,a_K)$, the
transformed representation and parameters
$\tilde{\mathbf z}'_{ij}=A\tilde{\mathbf z}_{ij}$,
$\bm\beta'=A^{-T}\bm\beta$, $\bm\gamma'_i=A^{-T}\bm\gamma_i$, and
$\bm\Sigma'=A^{-T}\bm\Sigma A^{-1}$ produce the same linear predictor and
marginal likelihood after integrating over the subject-specific effects.
Thus, without normalization, the scale and orientation of each learned latent
coordinate are not identifiable.

In our implementation, each scalar encoder output is centered and scaled to
unit empirical variance before posterior sampling. This fixes the shift and
positive scale conventions, but a sign ambiguity remains: replacing
$z_{ij,k}$ by $-z_{ij,k}$ and $(\beta_k,\gamma_{ik})$ by
$(-\beta_k,-\gamma_{ik})$ leaves the likelihood unchanged. Consequently, the
raw sign of $\beta_k$ is not an intrinsic scientific estimand. We therefore
focus on sign-invariant standardized summaries. Define the standardized
population-effect magnitude and standardized heterogeneity magnitude by
$\delta_k^{\mathrm{pop}}=|\beta_k|$ and
$\delta_k^{\mathrm{het}}=\sqrt{\Sigma_{kk}}$, and define their relative
contribution shares as
$\omega_k^{\mathrm{pop}}=\delta_k^{\mathrm{pop}}/\sum_{\ell=1}^{K}\delta_\ell^{\mathrm{pop}}$
and
$\omega_k^{\mathrm{het}}=\delta_k^{\mathrm{het}}/\sum_{\ell=1}^{K}\delta_\ell^{\mathrm{het}}$,
whenever the denominators are nonzero.

\begin{proposition}[Identifiable standardized latent summaries]
\label{prop:identifiable_estimands}
After centering and scaling the learned latent coordinates to unit empirical
variance, the summaries
$\delta_k^{\mathrm{pop}}$, $\delta_k^{\mathrm{het}}$,
$\omega_k^{\mathrm{pop}}$, and $\omega_k^{\mathrm{het}}$
are invariant to the remaining latent sign ambiguity. Consequently, posterior
comparisons such as
\[
    \Pr\!\left(
    \omega_g^{\mathrm{pop}}>\omega_h^{\mathrm{pop}}
    \mid \mathcal D
    \right)
    \quad\text{and}\quad
    \Pr\!\left(
    \omega_g^{\mathrm{het}}>\omega_h^{\mathrm{het}}
    \mid \mathcal D
    \right)
\]
are identifiable, whereas the raw sign of $\beta_k$ is not.
\end{proposition}

The proof is given in Appendix~\ref{app:identifiability}. This situation is
analogous to a standard regression setting in which predictors are observed
only up to an unknown scale factor: unstandardized coefficients cannot be
interpreted absolutely without fixing the predictor scale, so inference is
naturally based on standardized coefficients and relative effect sizes. In our
setting, the learned latent features are standardized before posterior
sampling, allowing us to compare standardized magnitudes across modalities.
However, because the sign of a learned latent coordinate remains arbitrary,
interpretation should focus on magnitude, relative importance, and posterior
rankings.

\subsection{Estimation}\label{subsec:estimation}

\noindent \textbf{Objective Function and Gradient Formulation.}
To perform Bayesian inference for the GLMM parameters $\Theta = \{\bm{\alpha},\bm{\beta}, \bm{\Sigma}\}$, where $\bm{\Sigma}$ represents the random effects covariance matrix, we aim to sample from the posterior distribution $p(\Theta | \mathcal{D})$. This is equivalent to minimizing the negative log-posterior potential function $f(\Theta)$, which decomposes into a prior component and a sum of subject-specific marginal likelihood contributions:$f(\Theta) = f_0(\Theta) + \sum_{i=1}^{N} f_i(\Theta)$,
where $f_0(\Theta) = -\log p(\Theta)$ represents the negative log-prior. The term $f_i(\Theta)$ denotes the negative marginal log-likelihood for subject $i$. In our deep probabilistic framework, the standardized latent features $\mathbf{z}_i$ extracted by the neural network are treated as covariates during the gradient derivation for the GLMM parameters. Thus, $f_i(\Theta)$ is defined as the integral over the random effects $\bm{\gamma}_i$:
$    f_i(\Theta) = -\log p(\mathbf{y}_i | \mathbf{z}_i, \Theta) = -\log \int p(\mathbf{y}_i | \mathbf{z}_i, \bm{\gamma}_i,\bm{\alpha}, \bm{\beta}) \, p(\bm{\gamma}_i | \bm{\Sigma}) \, d\bm{\gamma}_i.$
Directly optimizing or sampling based on $f_i(\Theta)$ is computationally intractable due to the high-dimensional integral. 

To overcome this, we follow \citep{berchuck2026scalablebayesianinferencegeneralized} and utilize Fisher's Identity, which expresses the gradient of the marginal log-likelihood as an expectation over the conditional posterior of the random effects:
$    g_i(\Theta) \triangleq \nabla f_i(\Theta) = \mathbb{E}_{p(\bm{\gamma}_i | y_i, \mathbf{z}_i, \Theta)} \left[ -\nabla_{\Theta} \log p(y_i, \bm{\gamma}_i | \mathbf{z}_i, \Theta) \right].$ Based on this identity, we construct an unbiased Monte Carlo estimator $\hat{g}_i(\Theta)$. We then draw samples $\{\bm{\gamma}_{ir}\}_{r=1}^R$ from the conditional posterior $p(\bm{\gamma}_i | y_i, \mathbf{z}_i, \Theta)$ via a Metropolis-Hastings (MH) step, yielding the estimator:
\begin{equation}
\label{eq:fisher_estimator}
    \hat{g}_i(\Theta) = -\frac{1}{R} \sum_{r=1}^{R} \nabla_{\Theta} \left( \log p(y_i | \mathbf{z}_i, \bm{\gamma}_{ir}, \bm{\alpha},\bm{\beta}) + \log p(\bm{\gamma}_{ir} | \bm{\Sigma}) \right).
\end{equation}

\noindent \textbf{Two-Stage Optimization Strategy.}
We separate supervised representation learning from Bayesian posterior
sampling. The encoder $\bm{\phi}$ is first trained together with the GLMM
likelihood to learn a predictive latent representation. Posterior samples are
then collected only after this representation is fixed, so that SGLD targets a
stationary conditional posterior for the GLMM parameters rather than a moving
target induced by continuously changing features. The full computational
procedure is summarized in Algorithm~\ref{alg:two_stage_sgld} in
Appendix~\ref{app:two_stage_algorithm}.

\textit{Phase 1: Supervised representation learning.}
During the first $K_{\mathrm{burn}}$ epochs, we update the encoder parameters
$\bm{\phi}$ using stochastic gradient methods, such as Adam, while updating
$\Theta$ in parallel with SGLD-type stochastic-gradient steps. This stage learns
a stable supervised feature map $h_{\bm{\phi}}(\cdot)$; samples generated during
this stage are treated as burn-in and are not used for posterior summaries.

\textit{Phase 2: Frozen conditional posterior sampling.}
Let $\hat{\bm{\phi}}$ denote the encoder obtained after Phase 1, and let
$\mathbf{z}_{\hat{\bm{\phi}},ij}
= \mathrm{Std}\{h_{\hat{\bm{\phi}}}(\mathbf{u}_{ij})\}$
be the frozen and standardized latent representation for observation $j$ from
subject $i$. After freezing $\hat{\bm{\phi}}$ and locking the normalization
statistics, the latent representation is treated as a fixed covariate in the
GLMM. Equivalently, SGLD targets the conditional posterior
$\pi_{\hat{\bm{\phi}}}(\Theta \mid \mathcal D)
\propto p(\Theta)\prod_{i=1}^N m_i(\Theta;\hat{\bm{\phi}})$,
where
$m_i(\Theta;\hat{\bm{\phi}})
=
\int
\{\prod_{j=1}^{n_i}
p(y_{ij}\mid \mathbf{x}_{ij},
\mathbf{z}_{\hat{\bm{\phi}},ij},
\gamma_i,\Theta)\}
p(\gamma_i\mid\Sigma)\,d\gamma_i$
is the subject-specific marginal likelihood after integrating out the random
effect $\gamma_i$. We then apply SGLD only to $\Theta$ and collect samples from
this fixed conditional posterior. Thus, uncertainty quantification is Bayesian
for the GLMM parameters conditional on the learned encoder
$\hat{\bm{\phi}}$; we do not claim to propagate posterior uncertainty over all
neural-network weights.

This two-stage procedure can therefore be viewed as conditional Bayesian
inference after joint supervised representation learning. Theoretical work
supports the use of trained feature maps in this role: mean-field analyses show
that SGD-trained neural networks can be described by nonlinear feature-learning
dynamics and can achieve near-ideal generalization in representative settings
\citep{mei2018meanfield}; related works further distinguish the
feature-learning/mean-field regime from linearized or fixed-feature regimes
\citep[Ch.~20]{spiliopoulos2025mathdl}. Conditioning on a jointly learned,
high-quality representation for downstream Bayesian inference is therefore a
principled practical compromise, analogous in spirit to empirical Bayes
procedures in which nuisance or hyperparameters are estimated first and the
parameters of interest are then given Bayesian treatment
\citep{robbins1956empirical, efron2010large}. In our setting,
Appendix~\ref{app:conditional_stability} provides a local perturbation argument
showing that, once the learned representation is stable, the resulting
conditional posterior mode and covariance are stable to small representation
errors. As detailed in Appendix~\ref{app:hyperparameter-tuning}, hyperparameter tuning
in the simulation and real-data studies also considered non-iterative training
variants and stage-1 learning-rate decay schedules (\texttt{lr\_decay}). Neither
alternative showed systematic empirical gains over the proposed two-stage
procedure, supporting the robustness of separating representation learning from
conditional posterior sampling in a stabilized feature space.

\noindent \textbf{SGLD Update and Parameter Evolution.}
The global parameters $\Theta$ are updated using SGLD. At iteration $t$ with a mini-batch $\mathcal{B}_t$ of size $S$, the update rule incorporates the gradient of the prior and the scaled sum of the stochastic gradients from the likelihood:
\begin{equation}
    \Theta_{t+1} = \Theta_t - \epsilon_t \left( \nabla f_0(\Theta_t) + \frac{N}{S} \sum_{i \in \mathcal{B}_t} \hat{g}_i(\Theta_t) \right) + \sqrt{2\epsilon_t} \bm{\eta}_t,
\end{equation}
where $\epsilon_t$ is the step size, $N$ is the total subject size, and $\bm{\eta}_t \sim \mathcal{N}(0, \mathbf{I})$ injects Gaussian noise.

\noindent \textbf{Post-hoc Variance Correction.}
Fixed-step SGLD produces samples with inflated variance because of minibatch
subsampling, Monte Carlo gradient estimation, and injected Langevin noise
\citep{berchuck2026scalablebayesianinferencegeneralized}. To obtain calibrated
uncertainty estimates, we apply a Lyapunov-equation-based post-hoc correction.
In brief, the correction estimates the covariance of the injected stochastic
noise, solves for the posterior precision matrix that would induce the observed
SGLD covariance, and linearly rescales the unconstrained samples. The full
correction equations and sample-rescaling map are given in
Appendix~\ref{app:variance_correction}.





\section{Simulation study}
Here we evaluate the proposed framework through numerical simulations. Our goals are to (1) assess whether the neural encoder can recover complex nonlinear latent structures, and (2) validate posterior inference and variance correction when LMMs or GLMMs include neural-network-learned latent features.

\subsection{LMM setting with random intercept} 

We first consider a semiparametric linear mixed model (LMM) with a random intercept and a neural latent feature:
\begin{equation}
\begin{split}
    y_{ij} = \beta_0 + \beta_1 z(\mathbf{x}_{ij}) + \beta_2 x_{2,ij} + \gamma_i + \epsilon_{ij},\\ \quad i=1,\dots,N, \; j=1,\dots,n_i
\end{split}
\end{equation}
The simulation includes three nonlinear latent structures: a complex one-dimensional function, a paraboloid, and a saddle surface. Full data-generating mechanisms, implementation details, and benchmark construction are provided in Appendix~\ref{app:simulation_lmm_details}.

The top panel of Figure~\ref{fig:simulation_results} visualizes the learned functions. It is important to note that the latent function in such semiparametric models is subject to \textit{scale and shift non-identifiability} (i.e., $h \approx a z + b$). Therefore, we compare the standardized vectors of the true $z$ and the learned $h$. In the 1D case (Fig.~\ref{fig:simulation_results}a), the learned curve almost perfectly overlaps with the true standardized function (correlation 0.9989). For the 2D cases (Fig.~\ref{fig:simulation_results}b-e), the model effectively captures the geometry of both shapes, with high correlations on standardized outputs (0.9986 and 0.9971), confirming that the network has identified the underlying mechanisms despite the complex likelihood function.

Posterior means and corrected posterior variances also agree with the Gibbs benchmark; the full LMM posterior validation plot is shown in Figure~\ref{fig:LMM_posterior_comparison} in Appendix~\ref{app:simulation_lmm_posterior}.

\subsection{GLMM setting with high-dimensional features and random slopes}
We next consider a high-dimensional logistic GLMM with both random intercepts and random slopes. The binary response follows a Bernoulli distribution with logits
\begin{equation}
    \operatorname{logit}(p_{ij}) =
    (\beta_0 + \gamma_{0i})
    + (\beta_1 + \gamma_{1i}) z(\mathbf{x}_{ij})
    + \beta_2 x_{2,ij}.
\end{equation}
Here $\mathbf{x}_{ij}\in\mathbb{R}^{20}$, and only 3 of the 20 input dimensions enter the true latent feature, with the remaining 17 dimensions serving as noise features. This makes the setting a joint test of nonlinear representation learning, implicit feature selection, and random-slope variance recovery. Full data-generating and implementation details are provided in Appendix~\ref{app:simulation_glmm_details}.

\begin{figure}[t]
    \centering
    \includegraphics[width=0.85\linewidth]{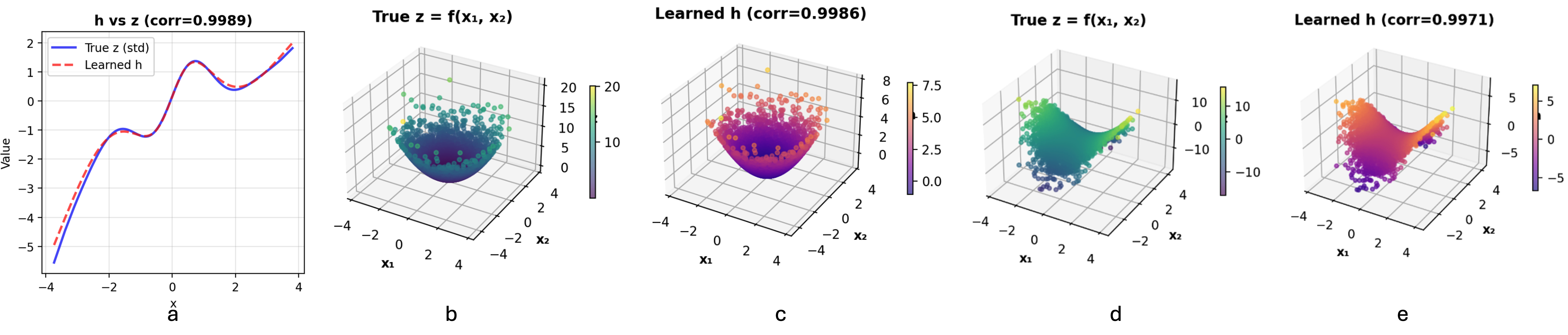}\\[2pt]
    \includegraphics[width=0.85\linewidth]{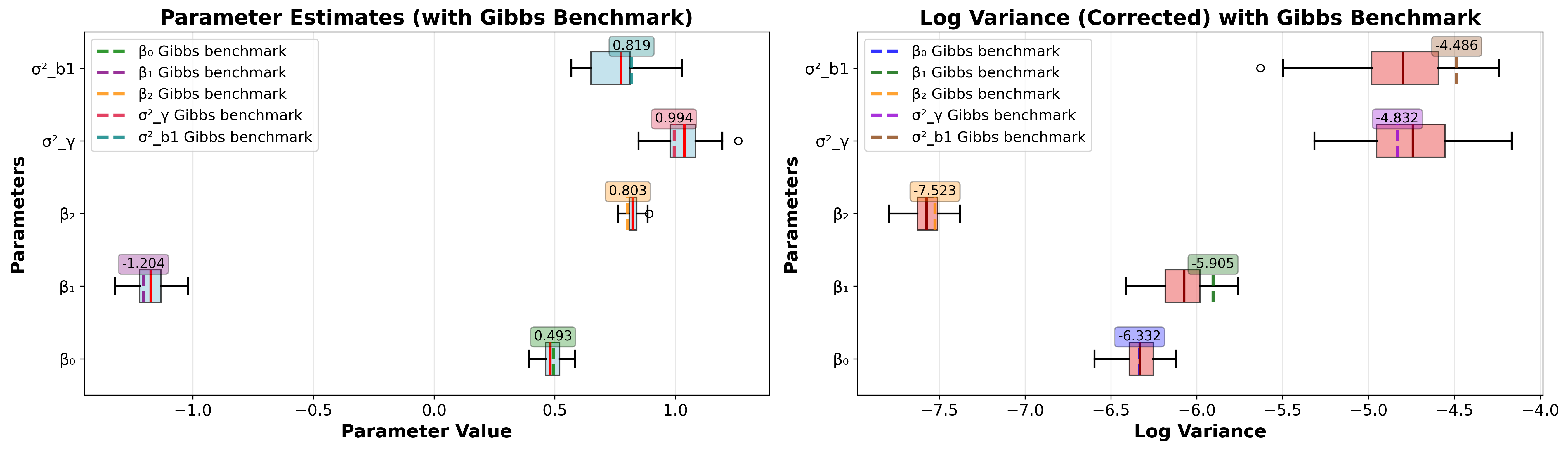}
    \caption{\textbf{Simulation results.} Top: latent structure recovery for the LMM setting, with subpanel (a) comparing the true and learned 1D nonlinear functions and subpanels (b-e) showing ground-truth and learned latent manifolds. Bottom: posterior inference in the high-dimensional GLMM setting, comparing parameter estimates (left) and corrected posterior log-variances (right) against the MCMC benchmark.}
    \label{fig:simulation_results}
\end{figure}
\paragraph{Results and analysis}
The bottom panel of Figure~\ref{fig:simulation_results} presents the posterior inference results. The left side demonstrates that our method accurately recovers the posterior means for all fixed effects and variance components, with boxplot medians closely matching the benchmark. Importantly, the right side validates our variance correction in this complex setting. Despite the high-dimensional input and additional noise introduced by the random slope, the corrected log-variances show that benchmark posterior variances fall within the range of our estimates.

The model retains the inherent sign and scale non-identifiability described in Section~\ref{subsec:identifiability}; therefore, the results presented here are evaluated after standardizing the learned latent representation to align with the ground truth scale, confirming that the structural relationships are correctly captured.

\section{Real data experiments}

We evaluated the framework in two longitudinal real-data settings: glaucoma progression from visual field images and adolescent mental health risk from multi-domain ABCD Study predictors.

\subsection{Glaucoma longitudinal image data analysis}

To demonstrate the practical utility of our framework in medical decision-making, we applied our method to forecast glaucoma progression using longitudinal visual field (VF) data. Early identification of eyes at high risk of rapid functional loss is critical for preserving vision \citep{berchuck2019estimating}.

We analyzed longitudinal standard automated perimetry (SAP) data from a glaucoma cohort. Full cohort construction, visual-field preprocessing, outcome definition, quality control, and split details are provided in Appendix~\ref{app:glaucoma_data}.

\paragraph{Model framework and implementation}
We adapted our Multimodal GLMM framework to integrate a Convolutional Neural Network (CNN) encoder with a logistic mixed-effects component. The CNN extracts a scalar latent feature $z_{ij}$ from the input image, and patient age is included as an explicit covariate. We model the progression probability $p_{ij}$ as:
$    \text{logit}(p_{ij}) = (\beta_0 + \gamma_{0i}) + (\beta_1 + \gamma_{1i}) z_{ij} + \beta_{\text{age}} \cdot \text{Age}_{ij}$
where $\text{Age}_{ij}$ is standardized. Random effects $\gamma_{0i} \sim \mathcal{N}(0, \sigma_{\gamma}^2)$ and $\gamma_{1i} \sim \mathcal{N}(0, \sigma_{b_1}^2)$ capture subject-specific baseline risk and sensitivity to visual field patterns, respectively. We estimated the posterior of $\Theta = \{\beta_0, \beta_1, \beta_{\text{age}}, \log \sigma_{\gamma}^2, \log \sigma_{b_1}^2\}$ via two-stage SGLD with variance correction; hyperparameter tuning is described in Appendix~\ref{app:hyperparameter-tuning}.
\paragraph{Results and interpretation}

\textbf{Population-level Effects ($\bm{\beta}$).} Posterior analysis confirms older age as a significant risk factor ($\hat{\beta}_{\text{age}} \approx 1.68$, 95\% CI: $[1.45, 1.95]$) \citep{berchuck2019estimating}.
The coefficient for the CNN-extracted latent feature is $\hat{\beta}_1 \approx -4.66$ (95\% CI: $[-5.10, -4.20]$). Since the outcome represents deterioration, this strong negative magnitude implies that a 1-SD increase in the learned latent feature $z$ corresponds to a substantial reduction in log-odds of progression. Our CNN learned a ``health index'' from raw visual fields, where higher values guard against functional loss.

\textbf{Subject-specific Heterogeneity ($\bm{\sigma}^2$).} The full posterior distributions, including the variance components, are shown in Figure~\ref{fig:glaucoma_all_posteriors} in Appendix~\ref{app:glaucoma_all_posteriors}. The large random intercept variance ($\sigma_{\gamma}^2 \approx 38$) reveals substantial baseline heterogeneity across patients (e.g., ``rapid progressors'' versus ``stable'' eyes) that is unexplained by age or image features alone. Additionally, the random slope variance ($\sigma_{b_1}^2 \approx 0.4$) indicates that the predictive value of visual field patterns varies by individual.

Finally, we compared our method against a baseline deterministic CNN trained without random effects. On the held-out test set, our Multimodal GLMM achieved an AUC of $0.764$, comparable to the baseline model (AUC: $0.761$). These results show incorporating rigorous Bayesian UQ and random effects enhances interpretability and explicitly quantifies individual deviations without compromising predictive accuracy. 

\subsection{Adolescent mental health risk analysis (ABCD Study)}
We then studied prediction of future mental health distress in the Adolescent Brain Cognitive Development (ABCD) Study. This challenges the model to integrate multi-domain developmental data while accounting for heterogeneity in how risk factors influence mental health trajectories.

\paragraph{Dataset and predictor domains}
We analyzed longitudinal observations from the ABCD Study \citep{karcher2021abcd,garavan2018recruiting}, using next-visit high-risk p-factor status as the binary outcome \citep{hill2025prediction}. Cohort construction, leakage control, preprocessing, and split details are provided in Appendix~\ref{app:abcd_grouping}. Predictors were organized according to clinical modifiability into eight neural-encoder domains: \textit{Screen Time \& Technology Use}, \textit{Sleep}, \textit{Family Environment}, \textit{Social \& School Environment}, \textit{Neighborhood \& Socioeconomic Context}, \textit{Neurobiology}, \textit{Family Psychiatric History}, and \textit{Fixed Biological/Demographic Characteristics}. Additional grouping rules, including the treatment of \texttt{total\_core} and excluded auxiliary variables, are also provided in Appendix~\ref{app:abcd_grouping}.

\paragraph{Model framework and implementation}
To disentangle domain-specific contributions, we employed eight parallel Multi-Layer Perceptron (MLP) encoders. Let $\mathcal{G}_{\mathrm{NN}}$ denote the eight neural-encoder domains listed above. For each domain $g \in \mathcal{G}_{\mathrm{NN}}$, an encoder $f_g(\cdot)$ maps the corresponding predictor block $X_{g,ij}$ to a scalar latent feature,
$z_{g,ij} = (f_g(X_{g,ij}) - \bar f_g)/s_g$, 
where $\bar f_g$ and $s_g$ are the training-set mean and standard deviation of the encoder output. This standardization makes the population-level coefficients comparable across domains. Let $c_{ij}$ denote the standardized \texttt{total\_core} trauma/adversity summary. The linear predictor is formulated as
\begin{equation}
    \operatorname{logit}(p_{ij}) =
    (\beta_0 + \gamma_{0i})
    + \beta_c c_{ij}
    + \sum_{g \in \mathcal{G}_{\mathrm{NN}}}
    (\beta_g + \gamma_{g,i}) z_{g,ij}.
    \label{eq:abcd_glmm}
\end{equation}
Here, $\beta_g$ represents the population-level association of domain $g$ with next-visit high-risk status, while $\gamma_{g,i} \sim \mathcal{N}(0,\sigma_g^2)$ captures subject-specific deviation in sensitivity to that domain. The random intercept $\gamma_{0i} \sim \mathcal{N}(0,\sigma_0^2)$ captures residual baseline heterogeneity across subjects. This distinguishes domains that are predictive on average from those whose effects vary between individuals.

\paragraph{Results and interpretation}
We focus the main interpretation on two quantities: the standardized population-level slopes $\beta_g$, which measure average domain-level associations, and the random-slope variances $\sigma_g^2$, which measure between-subject heterogeneity in domain-specific effects.

First, \textit{Sleep} showed the strongest population-level association with next-visit high-risk status ($\hat{\beta}_{\mathrm{Sleep}} \approx 0.95$, 95\% CI: $[0.91, 0.99]$). This agrees with the ABCD-based analysis of \citet{hill2025prediction}, which identified sleep disturbances as the most influential predictor of future high-risk psychopathology status, and with meta-analytic evidence linking shorter sleep duration to worse adolescent mood outcomes \citep{short2020sleep}. Recovering this known dominant signal supports the face validity of the proposed framework: even with learned domain-specific encoders and Bayesian mixed-effects inference, the model identifies sleep as the leading population-level risk domain.

Second, \textit{Neighborhood \& Socioeconomic Context} demonstrates why subject-specific effects matter. Its population-level association was relatively small ($\hat{\beta}_{\mathrm{NSES}} \approx -0.14$, 95\% CI: $[-0.18, -0.10]$), but its random-slope variance was the single largest variance component in the model ($\hat{\sigma}^2_{\mathrm{NSES}} \approx 0.06$, 95\% CI: $[0.05, 0.07]$), exceeding even the random intercept by posterior mean. Thus, a weak average effect may mask strong effects in particular subpopulations. This interpretation aligns with Moving to Opportunity evidence that neighborhood changes produced subgroup-dependent youth outcomes, with benefits for female youth offset by adverse effects for male youth \citep{kling2007experimental}. Our model therefore surfaces contextual effects that a single population-level coefficient could obscure. Additional domain-level rankings and interpretation are provided in Appendix~\ref{app:abcd_additional_results}. As with other learned coefficients, the sign of $\hat{\beta}_{\mathrm{NSES}}$ reflects the direction of the learned latent score rather than a direct monotone effect of neighborhood or socioeconomic variables.

To validate predictive integrity, we compared the Multi-Group Deep GLMM against a baseline MLP trained without random effects. On the held-out test set, our probabilistic model achieved an AUC of $0.78$ and Accuracy of $0.78$, performing comparably to the baseline model (AUC: $0.78$; Accuracy: $0.79$). This confirms that our framework successfully incorporates complex random effect structures to reveal subject-specific etiologies without sacrificing the high predictive performance characteristic of deep learning.

\section{Conclusion and discussion}

In this work, we presented a Multimodal Generalized Linear Mixed Model (GLMM) framework that enables scalable Bayesian inference for high-dimensional, longitudinal data. By integrating modality-specific neural network encoders with stochastic gradient MCMC (SGLD) inference, our approach overcomes computational bottlenecks of traditional Bayesian methods, yet preserves rigorous uncertainty quantification required for high-stakes applications. Through extensive simulations, we demonstrated that our method, augmented by a post-hoc variance correction, accurately recovers posterior distributions and avoids the variance underestimation typical of approximate inference.

We then applied our framework to glaucoma progression and adolescent mental health, yielding interpretable insights into both population-level trends and subject-specific heterogeneity. Our architecture deploys multiple neural encoders to extract features from distinct data domains/modalities and quantify their relative importance. When comparing child behavioral history versus neuroimaging, for example, we found that subject-specific risks are driven by differential sensitivities to specific environmental or behavioral factors. By retaining a mixed model structure, our method offers clinicians a transparent view of how individual patient trajectories deviate from population norms.

\paragraph{Limitations and future work}
Our framework has limitations that suggest future research directions. First, as noted in the simulation study, semiparametric models involving latent neural features are subject to sign and scale non-identifiability between the linear coefficients and the encoder outputs. This is an inherent limitation that we addressed by standardizing our neural network outputs, which allows structural relationships and relative effects to be identified.

Second, while our method quantifies the overall importance of each modality, 
it does not attribute this importance to specific raw input features. In glaucoma, for instance, we know visual fields are predictive but not which specific spatial defects are responsible. However, this is easily overcome by coupling our framework with established explainability methods, such as saliency maps \citep{simonyan2013deep} or concept bottleneck models \citep{koh2020concept}, which we will explore in future work.

\bibliographystyle{plainnat}
\bibliography{example_paper}

@article{mandel2023neural,
  title={Neural networks for clustered and longitudinal data using mixed effects models},
  author={Mandel, Francesca and Ghosh, Riddhi Pratim and Barnett, Ian},
  journal={Biometrics},
  volume={79},
  number={2},
  pages={711--721},
  year={2023},
  publisher={Wiley Online Library}
}

@article{simchoni2023integrating,
  title={Integrating random effects in deep neural networks},
  author={Simchoni, Giora and Rosset, Saharon},
  journal={Journal of Machine Learning Research},
  volume={24},
  number={156},
  pages={1--57},
  year={2023}
}

@article{tschalzev2024enabling,
  title={Enabling mixed effects neural networks for diverse, clustered data using Monte Carlo methods},
  author={Tschalzev, Andrej and Nitschke, Paul and Kirchdorfer, Lukas and L{\"u}dtke, Stefan and Bartelt, Christian and Stuckenschmidt, Heiner},
  journal={arXiv preprint arXiv:2407.01115},
  year={2024}
}

@inproceedings{shi2022generalized,
  title={Generalized deep mixed models},
  author={Shi, Jun and Jiang, Chengming and Gupta, Aman and Zhou, Mingzhou and Ouyang, Yunbo and Xiao, Qiang Charles and Song, Qingquan and Wu, Yi and Wei, Haichao and Gao, Huiji},
  booktitle={Proceedings of the 28th ACM SIGKDD Conference on Knowledge Discovery and Data Mining},
  pages={3869--3877},
  year={2022}
}

@inproceedings{welling2011bayesian,
  title={Bayesian learning via stochastic gradient Langevin dynamics},
  author={Welling, Max and Teh, Yee W},
  booktitle={Proceedings of the 28th international conference on machine learning (ICML-11)},
  pages={681--688},
  year={2011}
}

@article{tran2020bayesian,
  title={Bayesian deep net GLM and GLMM},
  author={Tran, M-N and Nguyen, Nghia and Nott, David and Kohn, Robert},
  journal={Journal of Computational and Graphical Statistics},
  volume={29},
  number={1},
  pages={97--113},
  year={2020},
  publisher={Taylor \& Francis}
}

@article{breslow1993approximate,
  title={Approximate inference in generalized linear mixed models},
  author={Breslow, Norman E and Clayton, David G},
  journal={Journal of the American statistical Association},
  volume={88},
  number={421},
  pages={9--25},
  year={1993},
  publisher={Taylor \& Francis}
}

@article{lopez2025uncertainty,
  title={Uncertainty Quantification for Machine Learning in Healthcare: A Survey},
  author={L{\'o}pez, L and Elsharief, Shaza and Jorf, Dhiyaa Al and Darwish, Firas and Ma, Congbo and Shamout, Farah E},
  journal={arXiv preprint arXiv:2505.02874},
  year={2025}
}

@article{lindstrom1990nonlinear,
  title={Nonlinear mixed effects models for repeated measures data},
  author={Lindstrom, Mary J and Bates, Douglas M},
  journal={Biometrics},
  pages={673--687},
  year={1990},
  publisher={JSTOR}
}

@article{gao2024explainable,
  title={An explainable longitudinal multi-modal fusion model for predicting neoadjuvant therapy response in women with breast cancer},
  author={Gao, Yuan and Ventura-Diaz, Sofia and Wang, Xin and He, Muzhen and Xu, Zeyan and Weir, Arlene and Zhou, Hong-Yu and Zhang, Tianyu and van Duijnhoven, Frederieke H and Han, Luyi and others},
  journal={Nature Communications},
  volume={15},
  number={1},
  pages={9613},
  year={2024},
  publisher={Nature Publishing Group UK London}
}

@inproceedings{gal2016dropout,
  title={Dropout as a bayesian approximation: Representing model uncertainty in deep learning},
  author={Gal, Yarin and Ghahramani, Zoubin},
  booktitle={international conference on machine learning},
  pages={1050--1059},
  year={2016},
  organization={PMLR}
}

@inproceedings{blundell2015weight,
  title={Weight uncertainty in neural network},
  author={Blundell, Charles and Cornebise, Julien and Kavukcuoglu, Koray and Wierstra, Daan},
  booktitle={International conference on machine learning},
  pages={1613--1622},
  year={2015},
  organization={PMLR}
}

@article{blei2017variational,
  title={Variational inference: A review for statisticians},
  author={Blei, David M and Kucukelbir, Alp and McAuliffe, Jon D},
  journal={Journal of the American statistical Association},
  volume={112},
  number={518},
  pages={859--877},
  year={2017},
  publisher={Taylor \& Francis}
}

@article{acosta2022multimodal,
  title={Multimodal biomedical AI},
  author={Acosta, Juli{\'a}n N and Falcone, Guido J and Rajpurkar, Pranav and Topol, Eric J},
  journal={Nature medicine},
  volume={28},
  number={9},
  pages={1773--1784},
  year={2022},
  publisher={Nature Publishing Group US New York}
}

@article{berchuck2019estimating,
  title={Estimating rates of progression and predicting future visual fields in glaucoma using a deep variational autoencoder},
  author={Berchuck, Samuel I and Mukherjee, Sayan and Medeiros, Felipe A},
  journal={Scientific Reports},
  volume={9},
  number={1},
  pages={18113},
  year={2019},
  publisher={Nature Publishing Group UK London}
}

@article{hill2025prediction,
  title={Prediction of mental health risk in adolescents},
  author={Hill, Elliot D and Kashyap, Pratik and Raffanello, Elizabeth and Wang, Yun and Moffitt, Terrie E and Caspi, Avshalom and Engelhard, Matthew and Posner, Jonathan},
  journal={Nature medicine},
  pages={1--7},
  year={2025},
  publisher={Nature Publishing Group US New York}
}

@article{garavan2018recruiting,
  title={Recruiting the ABCD sample: Design considerations and procedures},
  author={Garavan, Hugh and Bartsch, Hauke and Conway, K and Decastro, A and Goldstein, RZ and Heeringa, S and Jernigan, T and Potter, A and Thompson, W and Zahs, D},
  journal={Developmental cognitive neuroscience},
  volume={32},
  pages={16--22},
  year={2018},
  publisher={Elsevier}
}

@article{karcher2021abcd,
  title={The ABCD study: understanding the development of risk for mental and physical health outcomes},
  author={Karcher, Nicole R and Barch, Deanna M},
  journal={Neuropsychopharmacology},
  volume={46},
  number={1},
  pages={131--142},
  year={2021},
  publisher={Springer International Publishing Cham}
}

@article{simonyan2013deep,
  title={Deep inside convolutional networks: Visualising image classification models and saliency maps},
  author={Simonyan, Karen and Vedaldi, Andrea and Zisserman, Andrew},
  journal={arXiv preprint arXiv:1312.6034},
  year={2013}
}

@inproceedings{koh2020concept,
  title={Concept bottleneck models},
  author={Koh, Pang Wei and Nguyen, Thao and Tang, Yew Siang and Mussmann, Stephen and Pierson, Emma and Kim, Been and Liang, Percy},
  booktitle={International conference on machine learning},
  pages={5338--5348},
  year={2020},
  organization={PMLR}
}

@misc{berchuck2026scalablebayesianinferencegeneralized,
      title={Scalable Bayesian Inference for Generalized Linear Mixed Models via Stochastic Gradient MCMC}, 
      author={Samuel I. Berchuck and Youngsoo Baek and Felipe A. Medeiros and Andrea Agazzi},
      year={2026},
      eprint={2403.03007},
      archivePrefix={arXiv},
      primaryClass={stat.CO},
      url={https://arxiv.org/abs/2403.03007}, 
}

@article{short2020sleep,
  title={The relationship between sleep duration and mood in adolescents: A systematic review and meta-analysis},
  author={Short, Michelle A. and Booth, Stephen A. and Omar, Omar and Ostlundh, Linda and Arora, Teresa},
  journal={Sleep Medicine Reviews},
  volume={52},
  pages={101311},
  year={2020},
  doi={10.1016/j.smrv.2020.101311}
}

@article{yap2014parental,
  title={Parental factors associated with depression and anxiety in young people: A systematic review and meta-analysis},
  author={Yap, Marie Bee Hui and Pilkington, Pamela Doreen and Ryan, Siobhan Mary and Jorm, Anthony F.},
  journal={Journal of Affective Disorders},
  volume={156},
  pages={8--23},
  year={2014},
  doi={10.1016/j.jad.2013.11.007}
}

@article{kidger2012school,
  title={The effect of the school environment on the emotional health of adolescents: A systematic review},
  author={Kidger, Judi and Araya, Ricardo and Donovan, Jenny and Gunnell, David},
  journal={Pediatrics},
  volume={129},
  number={5},
  pages={925--949},
  year={2012},
  doi={10.1542/peds.2011-2248}
}

@article{reiss2013socioeconomic,
  title={Socioeconomic inequalities and mental health problems in children and adolescents: A systematic review},
  author={Reiss, Franziska},
  journal={Social Science \& Medicine},
  volume={90},
  pages={24--31},
  year={2013},
  doi={10.1016/j.socscimed.2013.04.026}
}

@article{kling2007experimental,
  title={Experimental analysis of neighborhood effects},
  author={Kling, Jeffrey R. and Liebman, Jeffrey B. and Katz, Lawrence F.},
  journal={Econometrica},
  volume={75},
  number={1},
  pages={83--119},
  year={2007},
  doi={10.1111/j.1468-0262.2007.00733.x}
}

@article{baek2026safe,
  title={Safe, Scalable, and Accurate Bayes Posterior Sampling for Large-Data Generalized Linear Mixed Models},
  author={Baek, Youngsoo and Berchuck, Samuel I},
  journal={arXiv preprint arXiv:2604.26029},
  year={2026}
}

@inproceedings{pituk2025bayesian,
  title={Do Bayesian neural networks actually behave like Bayesian models?},
  author={Pituk, G{\'a}bor and Shirvaikar, Vik and Rainforth, Tom},
  booktitle={Forty-second International Conference on Machine Learning},
  year={2025}
}

@article{mei2018meanfield,
  title   = {A mean field view of the landscape of two-layer neural networks},
  author  = {Mei, Song and Montanari, Andrea and Nguyen, Phan-Minh},
  journal = {Proceedings of the National Academy of Sciences},
  volume  = {115},
  number  = {33},
  pages   = {E7665--E7671},
  year    = {2018},
  doi     = {10.1073/pnas.1806579115}
}

@book{spiliopoulos2025mathdl,
  title     = {Mathematical Foundations of Deep Learning Models and Algorithms},
  author    = {Spiliopoulos, Konstantinos and Sowers, Richard B. and Sirignano, Justin},
  series    = {Graduate Studies in Mathematics},
  volume    = {252},
  publisher = {American Mathematical Society},
  year      = {2025},
  note      = {Chapter 20: Optimization in the Feature Learning Regime: Mean Field Scaling}
}

@incollection{robbins1956empirical,
  author    = {Robbins, Herbert E.},
  title     = {An Empirical Bayes Approach to Statistics},
  booktitle = {Proceedings of the Third Berkeley Symposium on Mathematical Statistics and Probability, Volume 1: Contributions to the Theory of Statistics},
  editor    = {Neyman, Jerzy},
  pages     = {157--164},
  publisher = {University of California Press},
  address   = {Berkeley, CA},
  year      = {1956}
}

@book{efron2010large,
  author    = {Efron, Bradley},
  title     = {Large-Scale Inference: Empirical Bayes Methods for Estimation, Testing, and Prediction},
  series    = {Institute of Mathematical Statistics Monographs},
  volume    = {1},
  publisher = {Cambridge University Press},
  year      = {2010},
  doi       = {10.1017/CBO9780511761362}
}

\newpage
\appendix

\section{Proof of Proposition~\ref{prop:identifiable_estimands}}
\label{app:identifiability}

We first restate the proposition from Section~\ref{subsec:identifiability}.

\begin{proposition}[Identifiable standardized latent summaries, restated]
After centering and scaling the learned latent coordinates to unit empirical
variance, the summaries
$\delta_k^{\mathrm{pop}}$, $\delta_k^{\mathrm{het}}$,
$\omega_k^{\mathrm{pop}}$, and $\omega_k^{\mathrm{het}}$
are invariant to the remaining latent sign ambiguity. Consequently, posterior
comparisons such as
\[
    \Pr\!\left(
    \omega_g^{\mathrm{pop}}>\omega_h^{\mathrm{pop}}
    \mid \mathcal D
    \right)
    \quad\text{and}\quad
    \Pr\!\left(
    \omega_g^{\mathrm{het}}>\omega_h^{\mathrm{het}}
    \mid \mathcal D
    \right)
\]
are identifiable, whereas the raw sign of $\beta_k$ is not.
\end{proposition}

\begin{proof}
We first establish the general latent-coordinate invariance. Consider the
neural-GLMM linear predictor
\[
    \eta_{ij}
    =
    \mathbf x_{ij}^{\top}\bm\alpha
    +
    \tilde{\mathbf z}_{ij}^{\top}(\bm\beta+\bm\gamma_i),
    \qquad
    \bm\gamma_i\sim N(\mathbf 0,\bm\Sigma),
\]
where
$\tilde{\mathbf z}_{ij}=(1,z_{ij,1},\ldots,z_{ij,K})^\top$
includes the intercept coordinate. Let $A$ be any nonsingular matrix acting on
the latent coordinates, with the intercept coordinate fixed when present.
Define
\[
    \tilde{\mathbf z}'_{ij}=A\tilde{\mathbf z}_{ij},
    \qquad
    \bm\beta'=A^{-T}\bm\beta,
    \qquad
    \bm\gamma'_i=A^{-T}\bm\gamma_i,
    \qquad
    \bm\Sigma'=A^{-T}\bm\Sigma A^{-1}.
\]
Then, for every observation,
\[
    (\tilde{\mathbf z}'_{ij})^{\top}(\bm\beta'+\bm\gamma'_i)
    =
    (A\tilde{\mathbf z}_{ij})^{\top}A^{-T}(\bm\beta+\bm\gamma_i)
    =
    \tilde{\mathbf z}_{ij}^{\top}(\bm\beta+\bm\gamma_i).
\]
Hence the conditional linear predictor, and therefore the conditional
likelihood, is unchanged after the corresponding change of variables
$\bm\gamma'_i=A^{-T}\bm\gamma_i$. Moreover, if
$\bm\gamma_i\sim N(\mathbf 0,\bm\Sigma)$, then
\[
    \bm\gamma'_i
    \sim
    N(\mathbf 0,A^{-T}\bm\Sigma A^{-1})
    =
    N(\mathbf 0,\bm\Sigma').
\]
Therefore,
\[
\int
    p(\mathbf y_i\mid \mathbf X_i,\tilde{\mathbf Z}_i,\bm\gamma_i,
    \bm\alpha,\bm\beta)
    p(\bm\gamma_i\mid\bm\Sigma)
    \,d\bm\gamma_i
=
\int
    p(\mathbf y_i\mid \mathbf X_i,\tilde{\mathbf Z}'_i,\bm\gamma'_i,
    \bm\alpha,\bm\beta')
    p(\bm\gamma'_i\mid\bm\Sigma')
    \,d\bm\gamma'_i.
\]
Thus the marginal likelihood is unchanged under latent-coordinate
transformations of this form.

Now impose the empirical standardization convention used before posterior
sampling. Each learned latent coordinate is centered and scaled to unit
empirical variance. This removes arbitrary shifts and positive rescalings.
The remaining ambiguity is a sign flip. In particular, under the transformation
\[
    z'_{ij,k}=-z_{ij,k},
    \qquad
    \beta'_k=-\beta_k,
    \qquad
    \gamma'_{ik}=-\gamma_{ik},
\]
the linear predictor remains unchanged. The corresponding variance component is
also unchanged:
\[
    \Sigma'_{kk}
    =
    \operatorname{Var}(\gamma'_{ik})
    =
    \operatorname{Var}(-\gamma_{ik})
    =
    \operatorname{Var}(\gamma_{ik})
    =
    \Sigma_{kk}.
\]
Therefore,
\[
    \delta_k^{\mathrm{pop}\,\prime}
    =
    |\beta'_k|
    =
    |-\beta_k|
    =
    |\beta_k|
    =
    \delta_k^{\mathrm{pop}},
\]
and
\[
    \delta_k^{\mathrm{het}\,\prime}
    =
    \sqrt{\Sigma'_{kk}}
    =
    \sqrt{\Sigma_{kk}}
    =
    \delta_k^{\mathrm{het}}.
\]
The same argument applies coordinate-wise to all latent dimensions. Hence the
denominators in
\[
    \omega_k^{\mathrm{pop}}
    =
    \frac{\delta_k^{\mathrm{pop}}}
    {\sum_{\ell=1}^{K}\delta_\ell^{\mathrm{pop}}},
    \qquad
    \omega_k^{\mathrm{het}}
    =
    \frac{\delta_k^{\mathrm{het}}}
    {\sum_{\ell=1}^{K}\delta_\ell^{\mathrm{het}}}
\]
are also unchanged under latent sign flips. It follows that
$\omega_k^{\mathrm{pop}}$ and $\omega_k^{\mathrm{het}}$ are invariant whenever
their denominators are nonzero.

Consequently, posterior comparisons such as
\[
    \Pr\!\left(
    \omega_g^{\mathrm{pop}}>\omega_h^{\mathrm{pop}}
    \mid \mathcal D
    \right)
    \quad\text{and}\quad
    \Pr\!\left(
    \omega_g^{\mathrm{het}}>\omega_h^{\mathrm{het}}
    \mid \mathcal D
    \right)
\]
depend only on sign-invariant summaries and are therefore identifiable. In
contrast, $\operatorname{sign}(\beta_k)$ changes under the same sign-flip
transformation and is not identifiable without an external orientation
convention.
\end{proof}

\section{Stability of Conditional Posterior Inference}
\label{app:conditional_stability}

Let
\[
    m_{\phi,N}(\Theta)
    =
    \frac{1}{N}
    \left[
    -\log p(\Theta)
    -
    \sum_{i=1}^N
    \log
    \int
    p(y_i\mid x_i,z_{\phi,i},\gamma_i,\Theta)
    p(\gamma_i\mid\Sigma)\,d\gamma_i
    \right]
\]
denote the average negative log-posterior potential induced by representation
$\phi$. Let $\phi_\star$ be an oracle or limiting representation and
$\hat\phi$ be the learned representation after Phase 1.

\begin{proposition}[Posterior stability under representation perturbation]
Assume that in a neighborhood $\mathcal N$ of the oracle posterior mode
$\Theta_\star$:
(i) $m_{\phi_\star,N}$ is twice continuously differentiable and
$\nabla^2 m_{\phi_\star,N}(\Theta)\succeq \lambda I$ for some $\lambda>0$;
(ii) the representation perturbation induces bounded score and curvature errors,
\[
    \sup_{\Theta\in\mathcal N}
    \left\|
    \nabla m_{\hat\phi,N}(\Theta)
    -
    \nabla m_{\phi_\star,N}(\Theta)
    \right\|
    \le r_N,
\]
and
\[
    \sup_{\Theta\in\mathcal N}
    \left\|
    \nabla^2 m_{\hat\phi,N}(\Theta)
    -
    \nabla^2 m_{\phi_\star,N}(\Theta)
    \right\|
    \le r_N .
\]
If $r_N<\lambda/2$ and both posterior modes lie in $\mathcal N$, then
\[
    \|\hat\Theta_{\hat\phi}-\hat\Theta_{\phi_\star}\|
    \le
    \frac{2r_N}{\lambda}.
\]
Moreover, under the usual local Gaussian approximation, the conditional
posterior covariance matrices satisfy
\[
    \left\|
    \{N\nabla^2 m_{\hat\phi,N}(\hat\Theta_{\hat\phi})\}^{-1}
    -
    \{N\nabla^2 m_{\phi_\star,N}(\hat\Theta_{\phi_\star})\}^{-1}
    \right\|
    =
    O\left(\frac{r_N}{N}\right).
\]
\end{proposition}

\begin{proof}
The first claim follows from strong convexity. Since
$\nabla m_{\hat\phi,N}(\hat\Theta_{\hat\phi})=0$ and
$\nabla m_{\phi_\star,N}(\hat\Theta_{\phi_\star})=0$,
a Taylor expansion of $\nabla m_{\hat\phi,N}$ around
$\hat\Theta_{\phi_\star}$ gives
\[
    \lambda_{\min}\{\nabla^2 m_{\hat\phi,N}(\bar\Theta)\}
    \|\hat\Theta_{\hat\phi}-\hat\Theta_{\phi_\star}\|
    \le
    \|\nabla m_{\hat\phi,N}(\hat\Theta_{\phi_\star})\|.
\]
By assumption,
$\lambda_{\min}\{\nabla^2 m_{\hat\phi,N}(\bar\Theta)\}\ge \lambda-r_N\ge
\lambda/2$, while
$\|\nabla m_{\hat\phi,N}(\hat\Theta_{\phi_\star})\|
=
\|\nabla m_{\hat\phi,N}(\hat\Theta_{\phi_\star})
-\nabla m_{\phi_\star,N}(\hat\Theta_{\phi_\star})\|
\le r_N$.
This proves the mode bound. The covariance statement follows from the matrix
identity
$A^{-1}-B^{-1}=A^{-1}(B-A)B^{-1}$ applied to the two Hessians, together with
the curvature perturbation bound and the fact that posterior covariance scales
as $N^{-1}$ under the local Gaussian approximation.
\end{proof}

\section{Estimation Algorithm and Variance Correction}
\label{app:estimation_algorithm}

\subsection{Post-Hoc Variance Correction}
\label{app:variance_correction}

Let $\boldsymbol{\theta}_k$ denote the unconstrained vector representation of a collected sample $\Theta_k \in \mathcal{S}_{\Theta}$, and let $\boldsymbol{\theta}^{\star}$ and $\Sigma_{\text{SGLD}}$ denote the empirical mean and covariance of these uncorrected samples. We first estimate the total noise covariance $\mathbf{\Gamma}$ injected by the stochastic gradient approximation:
\begin{equation}
\label{eq:gamma}
    \mathbf{\Gamma}(\boldsymbol{\theta}^{\star}) = \frac{\epsilon N^2}{2 S} \hat{\Psi}(\boldsymbol{\theta}^{\star}) + \kappa \mathbf{I},
\end{equation}
where $\hat{\Psi}(\boldsymbol{\theta}^{\star})$ is the covariance of the stochastic gradients and $\kappa \mathbf{I}$ represents the added Langevin diffusion noise. We relate the inflated covariance $\Sigma_{\text{SGLD}}$ to the true posterior precision matrix $A$ in this unconstrained parameterization using the Lyapunov equation:
\begin{equation}
\label{eq:lyapunov}
    A \Sigma_{\text{SGLD}} + \Sigma_{\text{SGLD}} A = 2 \mathbf{\Gamma}(\boldsymbol{\theta}^{\star}).
\end{equation}
We solve this linear equation for $A$ and compute the corrected covariance as $\Sigma_{\text{corr}} = A^{-1}$. To rescale the posterior samples, define the Cholesky decompositions $\Sigma_{\text{SGLD}} = \mathbf{E}^{\top}\mathbf{E}$ and $A = \mathbf{F}^{\top}\mathbf{F}$.
The corrected unconstrained samples are then given by
\begin{equation}
\label{eq:cholesky_correction}
    \boldsymbol{\theta}_k^{\text{corr}}
    =
    \mathbf{G}(\boldsymbol{\theta}_k - \boldsymbol{\theta}^{\star})
    + \boldsymbol{\theta}^{\star},
    \qquad
    \mathbf{G} = (\mathbf{E}^{\top}\mathbf{F})^{-1}.
\end{equation}
Finally, we map $\{\boldsymbol{\theta}_k^{\text{corr}}\}$ back to the original parameter space to obtain $\mathcal{S}_{\Theta}^{\text{corr}}$.

\subsection{Two-Stage SGLD Algorithm}
\label{app:two_stage_algorithm}

\begin{algorithm}[t]
\caption{Two-stage optimization for representation learning and SGLD with variance correction}
\label{alg:two_stage_sgld}
\begin{algorithmic}[1]
\State \textbf{Input:} Dataset $\mathcal{D}$, Burn-in $K_{\text{burn}}$, Total $K_{\text{total}}$, step sizes $\epsilon_t$ (sampling step $\epsilon$), correction constant $\kappa$.
\State \textbf{Initialize:} $\bm{\phi}$, $\Theta = \{\bm{\alpha}, \bm{\beta}, \bm{\Sigma}\}$, $\{\bm{\gamma}_i\}$, $\mathcal{S}_{\Theta} \leftarrow \emptyset$.

\For{epoch $k = 1$ \textbf{to} $K_{\text{total}}$}
    \State \textbf{Check Phase:} If $k > K_{\text{burn}}$, freeze $\bm{\phi}$ and fix $(\mu_z, \sigma_z)$.

    \For{mini-batch $\mathcal{B}_t \subset \mathcal{D}$ of size $S$}
        \State \textit{1) Feature Extraction:}
        \State $\mathbf{z}_i \leftarrow [1, \text{Standardize}(h_{\bm{\phi}}(\mathbf{u}_i))]$ for $i \in \mathcal{B}_t$.

        \State \textit{2) Random Effects (MH Step):}
        \State Sample $\bm{\gamma}_{ir} \sim p(\bm{\gamma}_i | \mathbf{y}_i, \mathbf{z}_i, \mathbf{x}_i, \Theta)$ (Inner Loop).

        \State \textit{3) Gradient Estimation:}
        \State Compute unbiased $\hat{g}_i(\Theta)$ via Fisher's Identity Eq.~\eqref{eq:fisher_estimator}.

        \State \textit{4) Parameter Updates:}
        \If{$k \leq K_{\text{burn}}$} \Comment{\textbf{Phase 1: NN Learning}}
            \State $\bm{\phi} \leftarrow \text{Adam}(\nabla_{\bm{\phi}} \mathcal{L}_{\text{joint}})$.
        \EndIf

        \State \textbf{SGLD Step:}
        \State $\Delta \Theta \leftarrow \nabla f_0(\Theta) + \frac{N}{S} \sum_{i \in \mathcal{B}_t} \hat{g}_i(\Theta)$
        \State $\Theta \leftarrow \Theta - \epsilon_t \Delta \Theta + \mathcal{N}(0, 2\epsilon_t \mathbf{I})$

        \If{$k > K_{\text{burn}}$} \Comment{\textbf{Phase 2: Sampling}}
            \State $\mathcal{S}_{\Theta} \leftarrow \mathcal{S}_{\Theta} \cup \{ \Theta \}$.
        \EndIf
    \EndFor
\EndFor
\Statex \textbf{Post-hoc Variance Correction:}
\State Compute posterior mean of unconstrained samples in $\mathcal{S}_{\Theta}$, named $\boldsymbol{\theta}^{\star}$.
\State Compute uncorrected posterior covariance, $\Sigma_{\text{SGLD}}$.
\State Compute injected-noise covariance, $\mathbf{\Gamma}(\boldsymbol{\theta}^{\star})$, using Eq.~\eqref{eq:gamma}.
\State Solve the Lyapunov equation in Eq.~\eqref{eq:lyapunov} for posterior precision $A$.
\State Compute corrected posterior covariance, $\Sigma_{\text{corr}} \leftarrow A^{-1}$.
\State Obtain corrected samples, $\mathcal{S}_{\Theta}^{\text{corr}}$, using Eq.~\eqref{eq:cholesky_correction}.
\end{algorithmic}
\end{algorithm}

\section{Simulation Details}
\label{app:simulation_details}

\subsection{LMM Simulation Details}
\label{app:simulation_lmm_details}

For the LMM simulation, $N=1000$ subjects are generated with $n_i=10$ observations each. The random effects are drawn from $\gamma_i \sim \mathcal{N}(0,\sigma_\gamma^2)$ and the measurement errors from $\epsilon_{ij} \sim \mathcal{N}(0,\sigma^2)$. The covariates $\mathbf{x}_{ij}$ and $x_{2,ij}$ are sampled from standardized normal distributions. The true parameters are $\beta_0=1.5$, $\beta_1=-0.5$, $\beta_2=0.6$, $\sigma^2=0.5$, and $\sigma_\gamma^2=1.5$.

The latent function $z(\cdot)$ represents the nonlinear structure to be learned by the neural encoder $h_{\phi}(\cdot)$. We consider three latent structures:
\begin{align}
    \text{Complex 1D:}\quad
    z(x)
    &=
    0.07x^3 - 0.2x^2
    + 1.5\sin(2x)e^{-0.3x^2}
    + 0.8 \mathcal{I}(x>0)\log(1+x)
    - 0.8 \mathcal{I}(x\le0)\log(1+|x|),\\
    \text{Paraboloid:}\quad
    z(x_1,x_2)
    &=
    x_1^2 + x_2^2,\\
    \text{Saddle:}\quad
    z(x_1,x_2)
    &=
    x_1^2 - x_2^2.
\end{align}

The neural network $h_{\phi}$ is implemented as a Multi-Layer Perceptron (MLP) with two hidden layers of 64 units each and GELU activation functions. We follow the two-stage training strategy described in Algorithm~\ref{alg:two_stage_sgld} in Appendix~\ref{app:two_stage_algorithm}. To account for the noise introduced by mini-batch subsampling ($S=50$) and the correlation between SGLD samples, we apply the post-hoc variance correction described in Appendix~\ref{app:variance_correction}.

We validate posterior summaries against a Gibbs sampler benchmark that yields exact inference sufficient to constitute ground truth for this simulation. This benchmark does not scale to larger datasets and requires fixing the latent features rather than learning them. Therefore, we implement the benchmark on the full dataset assuming the true form of $z$ is known.

\subsection{LMM Posterior Validation}
\label{app:simulation_lmm_posterior}

\begin{figure}[t]
    \centering
    \includegraphics[width=\linewidth]{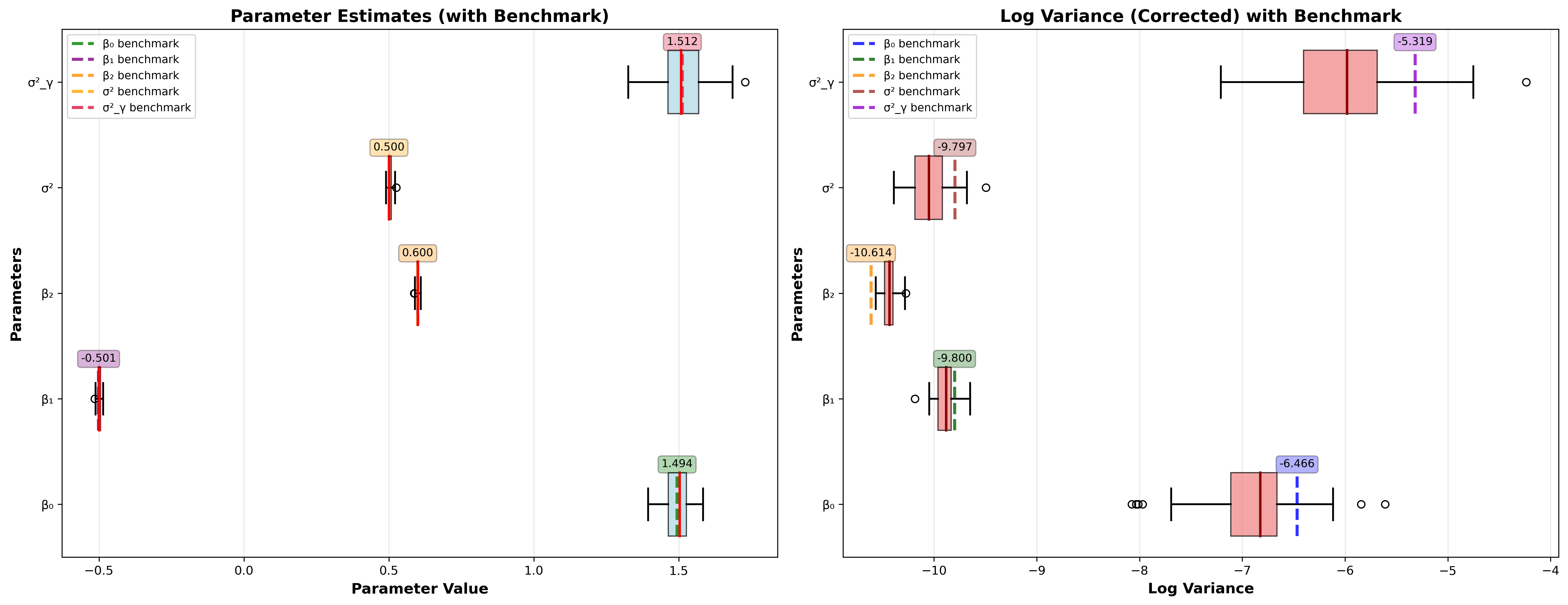}
    \caption{\textbf{Posterior validation.} Left: Posterior estimates of model parameters compared to ground truth benchmarks. Right: Comparison of posterior log-variances after applying variance correction against the Gibbs sampler benchmark. The corrected variances align with true posterior uncertainty.}
    \label{fig:LMM_posterior_comparison}
\end{figure}

\subsection{GLMM Simulation Details}
\label{app:simulation_glmm_details}

For the GLMM simulation, the input features satisfy $\mathbf{x}_{ij}\in\mathbb{R}^{20}$, but the true latent variable depends nonlinearly only on the first three dimensions:
\begin{equation}
    z(\mathbf{x}_{ij})
    =
    0.5x_{ij,1}^3
    - 2x_{ij,2}^2
    - 2x_{ij,3}.
\end{equation}
The remaining 17 dimensions are noise features. The binary response follows the logistic mixed model in the main text, with $\gamma_{0i}\sim\mathcal{N}(0,\sigma_{\gamma0}^2)$ denoting the random intercept and $\gamma_{1i}\sim\mathcal{N}(0,\sigma_{\gamma1}^2)$ denoting the random slope associated with the latent feature. The true parameters are $\beta_0=0.5$, $\beta_1=-1.2$, and $\beta_2=0.8$, with variance components $\sigma_{\gamma0}^2=1.0$ and $\sigma_{\gamma1}^2=0.8$.

The neural network $h_\phi$ maps the 20-dimensional input to a scalar latent variable using two hidden layers of 64 units with GELU activations. We apply the same two-stage training protocol as in the LMM setting. The variance correction is applied to the full set of five statistical parameters:
\[
    \boldsymbol{\theta}
    =
    \{
    \beta_0,\beta_1,\beta_2,
    \log\sigma_{\gamma0}^2,
    \log\sigma_{\gamma1}^2
    \}.
\]
We use an MCMC benchmark based on the ground-truth latent variables to evaluate posterior uncertainty.

\section{Hyperparameter Tuning}
\label{app:hyperparameter-tuning}

We tuned hyperparameters separately for the simulation and real-data experiments because the selection target is different in the two settings. In simulations, a Gibbs sampler was available as a reference, so we selected hyperparameters by comparing SGLD posterior summaries with the Gibbs posterior, focusing on posterior means, posterior variances, and representation quality metrics such as feature correlation. The most important tuned quantities were the SGLD step-size scale, the Monte Carlo budget for the random-effects sampler, the epoch at which posterior sampling begins, and the neural-encoder learning rate and weight decay.

We considered three training strategies for the neural mixed-model component. The first strategy, \emph{iterative learning}, is the main training procedure used in our primary experiments. In this setting, the neural encoder and the mixed-model parameters are updated jointly during the representation-learning phase, so that the learned latent features are repeatedly adapted to the current GLMM/LMM posterior estimates. After the sampling-start epoch, the encoder is fixed and SGLD is used to collect posterior samples for the statistical parameters. The second strategy, \emph{lr-decay}, keeps the same coupled training structure but linearly decays the neural learning rate to zero by a prespecified epoch, providing a smoother transition from representation learning to posterior sampling. The third strategy, \emph{non-iterative learning}, separates representation learning from posterior inference: the encoder is trained before posterior sampling begins, and the learned representation is then treated as fixed when running SGLD for the mixed-model parameters.

For all simulation experiments, stage 1 used 5 runs per configuration, stage 2 reran the top configurations with 15 runs, and the rank-1 configuration was rerun with 50 random seeds to estimate posterior means and posterior variances.

For the real-data experiments, no exact posterior reference was available. We therefore used five-fold validation performance to choose hyperparameters, ranking configurations primarily by validation AUC. After selecting the top-ranked configuration, we refit the model with the selected setting for the final posterior analysis. The real-data search focused on the neural learning rate, SGLD step size, neural weight decay, Metropolis-Hastings budget, proposal scale, sampling-start epoch, and whether the neural learning rate was kept constant or linearly decayed to zero.

All experiments were run on an internal SLURM-managed compute cluster.  The reported simulation runs used CPU workers only; no GPU was requested.  Individual simulation jobs were submitted with 4--8 CPU cores, 4--8GB memory, and a 48-hour wall-time limit, depending on the experiment.  The final real-data runs were also CPU jobs: Glaucoma and ABCD used 12 CPU cores and 16GB memory per run.

\begin{table}[t]
\centering
\caption{Hyperparameter grids for the simulation experiments. Entries list the candidate set used in tuning. When a two-stage search was used, stage 1 denotes the broad screening grid and stage 2 denotes reruns of the top-ranked configurations.}
\label{tab:sim_hyperparameter_tuning}
\small
\setlength{\tabcolsep}{2pt}
\renewcommand{\arraystretch}{1.18}
\resizebox{\textwidth}{!}{%
\begin{tabular}{p{2.6cm}p{2.35cm}p{2.35cm}p{2.35cm}p{2.35cm}p{2.35cm}p{2.35cm}}
\toprule
Model
& \multicolumn{3}{c}{LMM}
& \multicolumn{3}{c}{GLMM} \\
\cmidrule(lr){2-4}\cmidrule(lr){5-7}
Learning strategy
& Iterative
& LR decay
& Non-iterative
& Iterative
& LR decay
& Non-iterative \\
\midrule
Number of epochs
& Tuning: $\{160\}$; final: $2000$
& Tuning: $\{160\}$; final: $2000$
& Tuning: $\{160\}$; final: $2000$
& Tuning: $\{160\}$; final: $800$
& Tuning: $\{160\}$; final: $800$
& Tuning: $\{160\}$; final: $800$ \\
Sampling-start epoch
& $\{70,90,110\}$
& $\{70,90,110\}$
& $\{70,90,110\}$
& $\{90,110,130\}$
& $\{90,110,130\}$
& $\{90,110,130\}$ \\
Neural decay / stop epoch
& \textsc{NA}
& $\{50,70,90\}$
& \textsc{NA}
& \textsc{NA}
& $\{50,70,90\}$
& \textsc{NA} \\
Training MH total steps
& $\{1000,2000,3000\}$
& $\{1000,2000,3000\}$
& $\{1000,2000,3000\}$
& $\{1000,2000,3000\}$
& $\{1000,2000,3000\}$
& $\{1000,2000,3000\}$ \\
Training MH burn-in / kept samples
& $20\%/80\%$ of total steps
& $20\%/80\%$ of total steps
& $20\%/80\%$ of total steps
& $20\%/80\%$ of total steps
& $20\%/80\%$ of total steps
& $20\%/80\%$ of total steps \\
Training MH proposal sd
& $\{0.8,1.5,1.8\}$
& $\{0.8,1.5,1.8\}$
& $\{0.8,1.5,1.8\}$
& $\{0.8,1.5,1.8\}$
& $\{0.8,1.5,1.8\}$
& $\{0.8,1.5,1.8\}$\\
Variance-correction MH total steps
& $\{1000,2000,3000\}$
& $\{1000,2000,3000\}$
& $\{1000,2000,3000\}$
& $\{1000,2000,3000\}$
& $\{1000,2000,3000\}$
& $\{1000,2000,3000\}$ \\
Variance-correction MH proposal sd
& $\{0.8,1.5,1.8\}$
& $\{0.8,1.5,1.8\}$
& $\{0.8,1.5,1.8\}$
& $\{0.8,1.5,1.8\}$
& $\{0.8,1.5,1.8\}$
& $\{0.8,1.5,1.8\}$\\
SGLD step-size scale $\delta$
& $\{0.7,0.8,0.9,1.0\}$
& $\{0.7,0.8,0.9,1.0\}$
& $\{0.7,0.8,0.9,1.0\}$
& $\{0.7,0.8,0.9,1.0\}$
& $\{0.7,0.8,0.9,1.0\}$
& $\{0.7,0.8,0.9,1.0\}$ \\
Neural learning rate
& $\{5\times10^{-4},10^{-3}\}$
& $\{10^{-3}\}$
& $\{10^{-3}\}$
& $\{3\times10^{-4},10^{-3},3\times10^{-3}\}$ 
& $\{3\times10^{-4},10^{-3},3\times10^{-3}\}$ 
& $\{3\times10^{-4},10^{-3},3\times10^{-3}\}$ \\
Neural weight decay
& $\{0\}$
& $\{0\}$
& $\{0\}$
& $\{0,0.05,0.2,0.3\}$
& $\{0,0.05,0.2,0.3\}$
& $\{0,0.05,0.2,0.3\}$ \\
Neural updates per batch
& $\{2\}$
& $\{2\}$
& \textsc{NA}
& $\{1,2\}$
& $\{1,2\}$
& \textsc{NA} \\
\bottomrule
\end{tabular}}
\end{table}

\begin{table}[t]
\centering
\caption{Hyperparameter grids for the real-data experiments. Both real-data analyses used the neural GLMM model. The selected configuration is the rank-1 setting from five-fold cross-validation.}
\label{tab:real_hyperparameter_tuning}
\small
\setlength{\tabcolsep}{5pt}
\renewcommand{\arraystretch}{1.15}
\begin{tabular}{p{4.2cm}p{5.2cm}p{5.2cm}}
\toprule
Hyperparameter & Glaucoma & ABCD \\
\midrule
Selection criterion
& Five-fold validation AUC
& Five-fold validation AUC \\
Number of CV folds
& $5$
& $5$ \\
CV epochs / final epochs
& $125 / 1000$
& $125 / 2000$ \\
Neural learning rate
& $\{10^{-3},3\times10^{-3},3\times10^{-4}\}$
& $\{10^{-3},5\times10^{-3},3\times10^{-4}\}$ \\
SGLD step size $\epsilon$
& $\{10^{-4},3\times10^{-4},3\times10^{-5}\}$
& $\{3\times10^{-4},3\times10^{-5}\}$ \\
Neural weight decay
& $\{10^{-3},3\times10^{-3},10^{-2}\}$
& $\{10^{-3},5\times10^{-3},1.5\times10^{-2},5\times10^{-2}\}$ \\
MH total steps
& $\{1000,2000\}$
& $\{1000,2000\}$ \\
MH burn-in / kept samples
& $20\%/80\%$ of total steps
& $20\%/80\%$ of total steps \\
MH proposal sd
& $\{0.8,1.8\}$
& $\{0.8,1.8\}$ \\
Sampling-start epoch
& $\{40,70,100\}$
& $\{40,70,100\}$ \\
Neural learning-rate schedule
& $\{\texttt{constant},\texttt{lr\_decay}\}$
& $\{\texttt{constant},\texttt{lr\_decay}\}$ \\
Neural decay-end epoch
& $\{40,75\}$ for \texttt{lr\_decay}; \textsc{NA} for \texttt{constant}
& $\{25,50,75\}$ for \texttt{lr\_decay}; \textsc{NA} for \texttt{constant} \\
Minibatch size
& $50$ subjects
& $100$ subjects \\
\bottomrule
\end{tabular}
\end{table}

\section{Real-Data Cohort Construction and Preprocessing}
\label{app:real_data_preprocessing}

\subsection{Glaucoma Dataset and Preprocessing}
\label{app:glaucoma_data}

We analyzed longitudinal standard automated perimetry (SAP) data from a XX cohort study. The dataset comprises Humphrey 24-2 SAP tests (52 test locations) with functional loss quantified by Total Deviation (TD) values. To ensure compatibility with the neural encoder, non-rectangular visual fields were zero-padded to a $12 \times 12$ grid and normalized. Strict quality control excluded tests with artifacts or low reliability indices ($>33\%$ fixation losses or $>15\%$ false positives), resulting in 29,161 usable fields from 3,832 eyes (mean follow-up: 4.95 years, 7.61 visits).

The binary outcome $y_{ij}$, representing ``progression,'' was derived via Ordinary Least Squares (OLS) regression of the global visual field index over time. Progression was defined as a statistically significant negative slope at $\alpha=0.10$.

We implemented a 5-fold cross-validation strategy with a hybrid subject-temporal split to mimic clinical forecasting: 70\% of subjects were assigned to training, while the longitudinal visits of the remaining 30\% were split temporally (first 40\% history, middle 30\% validation, final 30\% future forecasting).

\subsection{ABCD Study Cohort, Preprocessing, and Predictor Grouping}
\label{app:abcd_grouping}

This subsection describes the cohort construction, preprocessing, and predictor grouping used in the ABCD Study analysis. The grouping was designed to preserve clinically interpretable domains while reflecting data availability after leakage control and missingness filtering.

We utilized data from the ABCD Study, a large-scale longitudinal cohort of 11,880 children recruited across the United States \citep{karcher2021abcd}. The dataset includes baseline assessments (ages 9--10) and annual follow-up visits collecting psychosocial, demographic, environmental, and neurobiological data \citep{garavan2018recruiting}. General psychopathology was characterized using the ``p-factor,'' a dimensional score derived from the Child Behavior Checklist (CBCL) syndrome scales \citep{hill2025prediction}. In our analytic sample, after requiring a non-missing next-visit p-factor quartile, 26,787 observations from 7,447 participants were available.

The binary outcome $y_{ij}$ was defined as ``High Risk'' status, corresponding to the top quartile of the p-factor score at the subsequent visit. Consistent with the glaucoma analysis, we employed a hybrid subject-temporal split: 70\% of subjects were assigned to training, while 30\% were split temporally into history, validation, and future forecasting subsets.

\paragraph{Leakage control and preprocessing}
To avoid circular prediction, all variables beginning with \texttt{CBCL\_} were removed from the predictor set before grouping. We also removed direct outcome and leakage columns, including current or next p-factor quartile variables, subject identifiers, and interview dates. Auxiliary columns such as \texttt{eventname}, \texttt{interview\_age}, \texttt{interview\_year}, and split indicators were assigned to an excluded ``Others'' category.

For each training split, variables with more than 25\% missingness in the training data were excluded. Remaining variables were median-imputed and standardized using training-set statistics, with the same transformations applied to validation and held-out test observations. The final retained dimensions were 20 screen-time variables, 26 sleep variables, 32 family-environment variables, 18 social/school variables, 7 neighborhood/socioeconomic variables, 33 neurobiology variables, 8 family-psychiatric-history variables, 4 fixed biological/demographic variables, and one linear trauma/adversity summary variable.

\paragraph{Clinical modifiability principle}
Variables were organized by approximate clinical modifiability. The first four domains represent potentially modifiable behavioral and proximal environmental factors. Neighborhood and socioeconomic variables were treated as somewhat modifiable because they may be influenced by family, school, health-system, or policy interventions but are less directly modifiable at the individual clinical level. Neurobiology and trauma/adversity history were treated as minimally modifiable in the short term. Family psychiatric history and fixed biological/demographic characteristics were treated as not modifiable.

\begin{table}[t]
\centering
\small
\caption{Final ABCD predictor grouping used in the real-data analysis.}
\label{tab:abcd_grouping}
\begin{tabularx}{\linewidth}{P{0.24\linewidth}P{0.17\linewidth}P{0.34\linewidth}P{0.12\linewidth}P{0.07\linewidth}}
\toprule
\textbf{Final domain} & \textbf{Modifiability} & \textbf{Assignment rule / key variables} & \textbf{Model role} & \textbf{\# Vars} \\
\midrule
Screen Time \& Technology Use
& Highly modifiable
& Prefixes such as \texttt{screen*}, \texttt{screentime*}, and \texttt{stq\_*}
& MLP encoder
& 20 \\

Sleep
& Highly modifiable
& Variables beginning with \texttt{sleepdisturb*} or containing \texttt{sleep}
& MLP encoder
& 26 \\

Family Environment
& Modifiable
& Prefixes such as \texttt{fam\_enviro*}, \texttt{fes\_*}, \texttt{parent\_monitor*}, and \texttt{parent\_rules*}
& MLP encoder
& 32 \\

Social \& School Environment
& Modifiable
& Prefixes such as \texttt{school\_*} and \texttt{prosocial\_*}
& MLP encoder
& 18 \\

Neighborhood \& Socioeconomic Context
& Somewhat modifiable
& Neighborhood and built-environment variables such as \texttt{neighborhood*} and \texttt{adi\_*}; crime, violence, safety, and deprivation variables; household income and parent education variables
& MLP encoder
& 7 \\

Neurobiology
& Minimally modifiable
& Neuroimaging and neurobiological prefixes including \texttt{rsfmri\_*}, \texttt{tfmri\_*}, \texttt{dmdtifp*}, \texttt{dmri*}, \texttt{smri*}, \texttt{mrisdp*}, \texttt{fmri\_*}, and \texttt{mri\_*}
& MLP encoder
& 33 \\

Family Psychiatric History
& Not modifiable
& Prefixes such as \texttt{fam\_history\_*}, \texttt{famhx\_*}, and \texttt{family\_mental\_health*}
& MLP encoder
& 8 \\

Fixed Biological/Demographic Characteristics
& Not modifiable
& Race/ethnicity, sex at birth, gender, and related demographic prefixes such as \texttt{demo\_sex\_*}, \texttt{sex\_*}, \texttt{race\_*}, and \texttt{ethnicity\_*}
& MLP encoder
& 4 \\

Trauma \& Adversity Summary
& Minimally modifiable
& \texttt{total\_core}; retained as the only trauma/adversity summary after preprocessing
& Linear covariate
& 1 \\

Others
& Excluded
& Auxiliary variables, split indicators, interview metadata, and variables not used in the GLMM-NN predictor set
& Excluded
& 8 \\
\bottomrule
\end{tabularx}
\end{table}

\paragraph{Modeling implications}
The final model uses neural encoders only for the eight multi-variable domains. The trauma/adversity summary \texttt{total\_core} is modeled as a direct standardized linear covariate because it is a single retained summary variable rather than a high-dimensional predictor block. The neighborhood/built-environment and socioeconomic categories are merged into a single \textit{Neighborhood \& Socioeconomic Context} encoder to improve statistical power and avoid fitting separate neural encoders to very small predictor groups.

\section{Additional ABCD Results and Interpretation}
\label{app:abcd_additional_results}

This appendix provides the full domain-level interpretation for the ABCD Study analysis. Because all neural encoder outputs and the \texttt{total\_core} covariate are standardized, coefficient magnitudes are directly comparable across domains. The signs should be interpreted as effects of learned latent features rather than direct monotone effects of every raw variable in a group; for example, the positive coefficient for the family-environment encoder indicates that larger values of the learned family-environment summary are associated with higher risk.

\paragraph{Population-level domain ranking}
As described in the main text, the strongest population-level association was observed for \textit{Sleep} ($\hat{\beta}_{\mathrm{Sleep}} \approx 0.95$, 95\% CI: $[0.91, 0.99]$). The next largest population-level associations were \textit{Social \& School Environment} ($\hat{\beta}_{\mathrm{SocialSchool}} \approx 0.57$, 95\% CI: $[0.53, 0.60]$) and \textit{Family Environment} ($\hat{\beta}_{\mathrm{FamilyEnv}} \approx 0.51$, 95\% CI: $[0.47, 0.55]$). Although these non-sleep domains are not emphasized in the main text, their ranking suggests that a substantial portion of predictive signal is concentrated in modifiable or proximal developmental contexts. Prior systematic reviews similarly identify modifiable parental factors, including lower warmth, conflict, over-involvement, and aversive parenting, as correlates of youth depression and anxiety \citep{yap2014parental}, and school connectedness and teacher support as prospective markers of adolescent emotional health \citep{kidger2012school}.

Moderate positive associations were found for \textit{Family Psychiatric History} ($\hat{\beta}_{\mathrm{FamPsych}} \approx 0.25$, 95\% CI: $[0.21, 0.29]$), the \texttt{total\_core} summary ($\hat{\beta}_{c} \approx 0.18$, 95\% CI: $[0.14, 0.22]$), and \textit{Screen Time \& Technology Use} ($\hat{\beta}_{\mathrm{Screen}} \approx 0.13$, 95\% CI: $[0.09, 0.17]$). Two domains showed small negative associations, \textit{Neurobiology} ($\hat{\beta}_{\mathrm{Neuro}} \approx -0.15$, 95\% CI: $[-0.19, -0.11]$) and \textit{Neighborhood \& Socioeconomic Context} ($\hat{\beta}_{\mathrm{NSES}} \approx -0.14$, 95\% CI: $[-0.18, -0.10]$), while \textit{Fixed Biological/Demographic Characteristics} had the smallest population-level effect ($\hat{\beta}_{\mathrm{BioDemo}} \approx 0.08$, 95\% CI: $[0.04, 0.12]$). The small, near-null neurobiology effect is consistent with \citet{hill2025prediction}, who found that adding neuroimaging did not improve predictive performance beyond psychosocial predictors. The socioeconomic-context finding should not be interpreted as absence of relevance; extensive evidence links lower socioeconomic status to higher child and adolescent mental health burden \citep{reiss2013socioeconomic}, but in our model this domain contributes more strongly through subject-specific heterogeneity than through a single average population-level slope.

\paragraph{Subject-specific heterogeneity}
The full ABCD posterior distributions are shown in Figure~\ref{fig:abcd_all_posteriors} in Appendix~\ref{app:abcd_all_posteriors}; these include the random-intercept and domain-specific random-slope variance components. In the ABCD Study, the largest variance component was the random slope for \textit{Neighborhood \& Socioeconomic Context} ($\hat{\sigma}^2_{\mathrm{NSES}} \approx 0.06$, 95\% CI: $[0.05, 0.07]$), which exceeded even the random intercept ($\hat{\sigma}_0^2 \approx 0.04$, 95\% CI: $[0.03, 0.04]$); the random intercept still reflects residual between-subject baseline risk not fully explained by the observed domains. The remaining domain-specific random-slope variances were smaller and comparable to one another (all $\hat{\sigma}^2 \lesssim 0.04$), indicating that, apart from neighborhood context, between-subject heterogeneity in domain sensitivity was relatively modest. This pattern highlights a separation between two forms of importance: modifiable domains such as \textit{Sleep}, \textit{Family Environment}, and \textit{Social \& School Environment} carry strong average population-level signals, whereas \textit{Neighborhood \& Socioeconomic Context} contributes disproportionately through subject-specific heterogeneity rather than through its population-level slope.

\section{Posterior of all parameters in glaucoma study}
\label{app:glaucoma_all_posteriors}
\begin{figure}
    \centering
    \includegraphics[width=\linewidth]{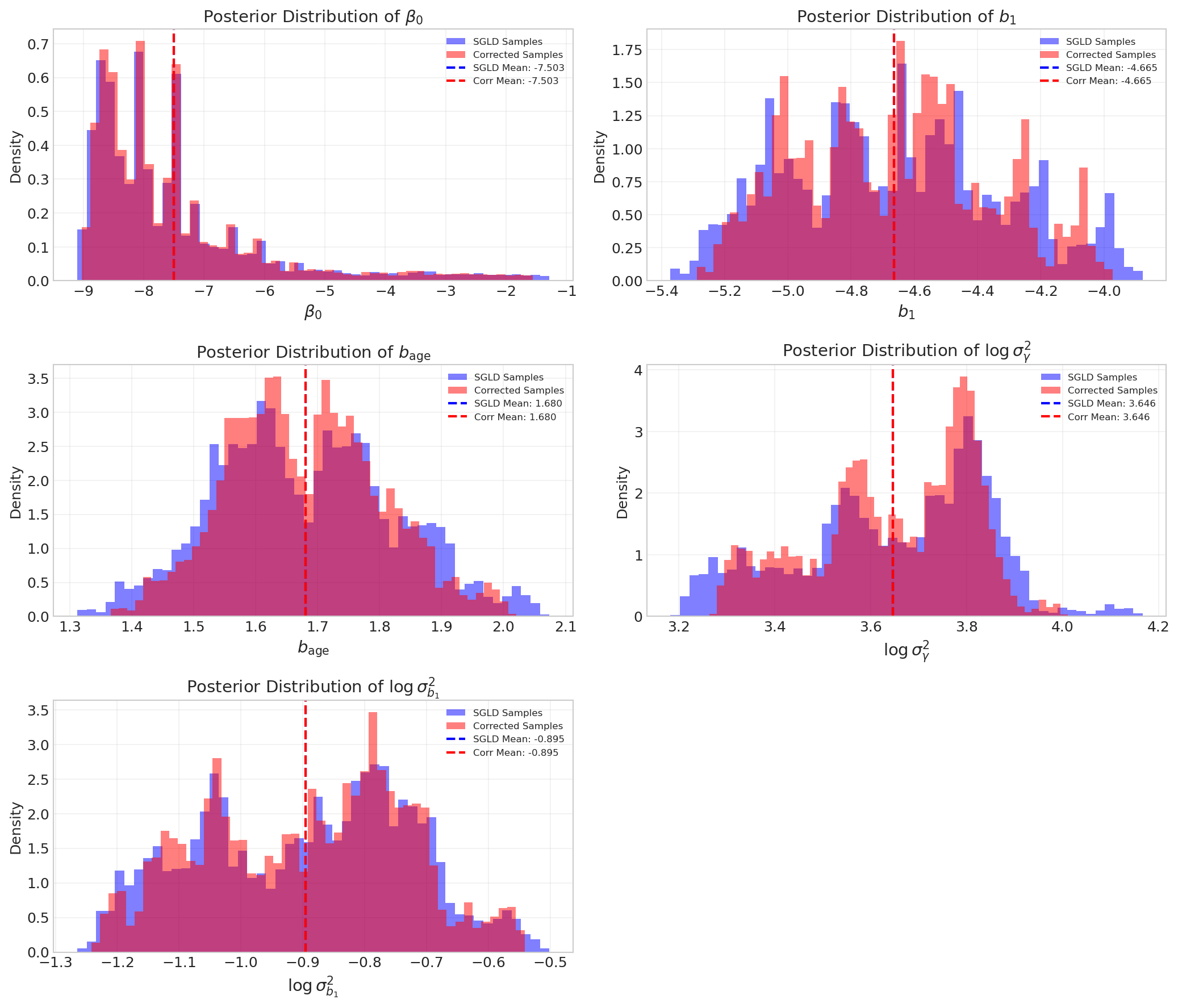}
    \caption{Posterior distribution of all parameters}
    \label{fig:glaucoma_all_posteriors}
\end{figure}

\section{Posterior of all parameters in ABCD study}
\label{app:abcd_all_posteriors}
\begin{figure}
    \centering
    \includegraphics[width=0.85\linewidth]{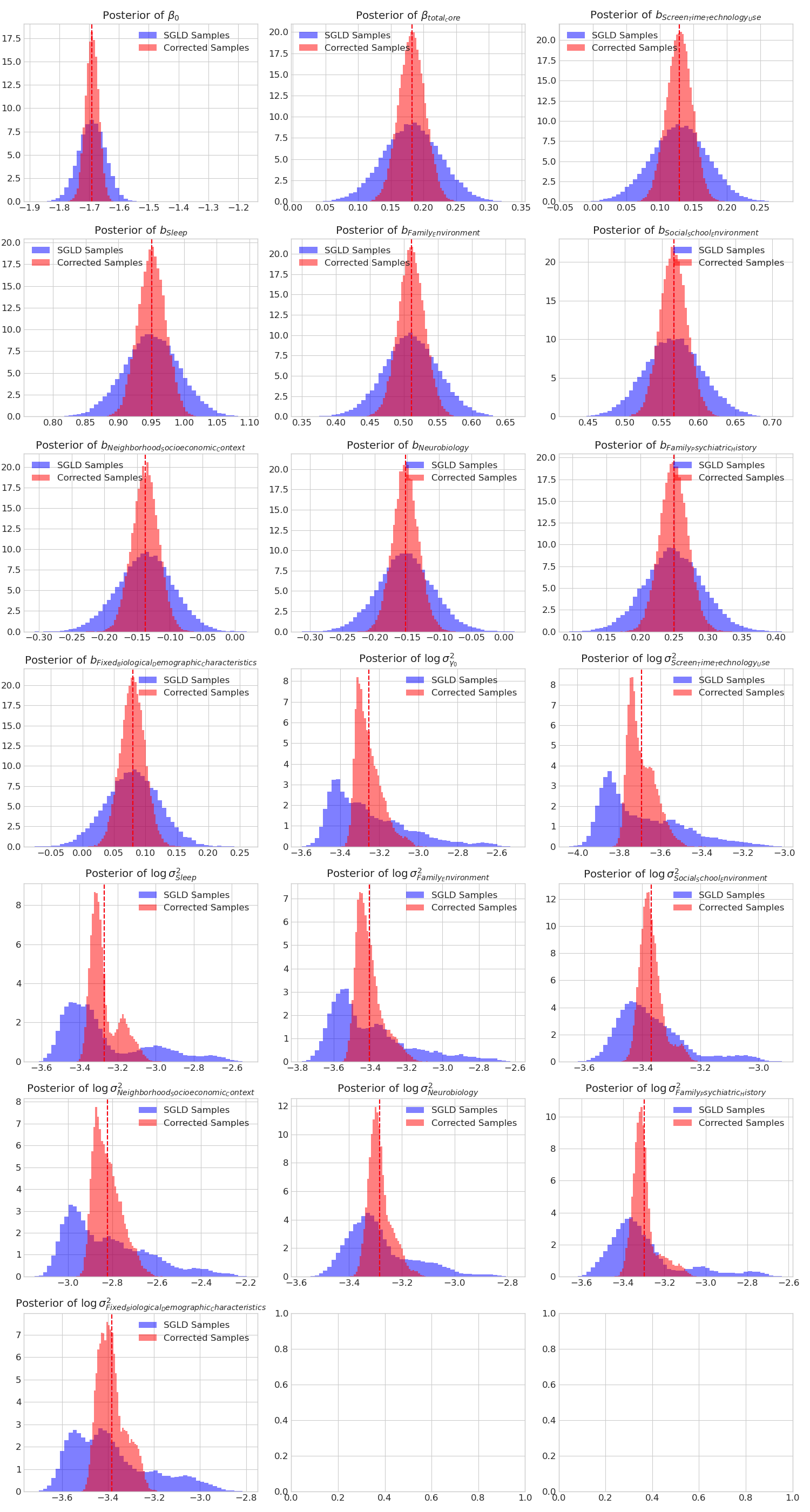}
    \caption{Posterior distribution of all parameters}
    \label{fig:abcd_all_posteriors}
\end{figure}

\clearpage


\end{document}